\documentclass{article}

    \PassOptionsToPackage{numbers}{natbib}
    \usepackage[preprint]{neurips_2025}

\usepackage{algorithm}
\usepackage{algorithmic}
\usepackage{multicol}

\usepackage[utf8]{inputenc} 
\usepackage[T1]{fontenc}    
\usepackage{hyperref}       
\usepackage{url}            
\usepackage{booktabs}       
\usepackage{amsfonts}       
\usepackage{nicefrac}       
\usepackage{microtype}      
\usepackage{xcolor}         
\definecolor{revisionblue}{RGB}{0,70,170}
\newcommand{\rev}[1]{{\color{revisionblue}#1}}
\newenvironment{revision}{\color{revisionblue}}{}

\usepackage{amsmath}
\usepackage{amssymb}
\usepackage{mathtools}
\usepackage{amsthm}

\theoremstyle{plain}
\newtheorem{theorem}{Theorem}[section]
\newtheorem{proposition}[theorem]{Proposition}
\newtheorem{lemma}[theorem]{Lemma}

\theoremstyle{definition}

\theoremstyle{remark}
\newtheorem{remark}[theorem]{Remark}

\usepackage{listings}

\title{Learning Cocoercive Conservative Denoisers via Helmholtz Decomposition for Poisson Imaging Inverse Problems}

\author{%
Deliang Wei \\
  School of Mathematical Sciences, \\
   East China Normal University,\\
   Shanghai 200241, China\\
   \texttt{52215500006@stu.ecnu.edu.cn} \\
   \And
   Peng Chen \\
   School of Mathematical Sciences, \\
   East China Normal University,\\
   Shanghai 200241, China\\
   \texttt{52265500005@stu.ecnu.edu.cn} \\
   \And
   Haobo Xu \\
   School of Mathematical Sciences, \\
   East China Normal University,\\
   Shanghai 200241, China\\
   \texttt{52275500009@stu.ecnu.edu.cn} \\
   \And
   Jiale Yao \\
   School of Mathematical Sciences, \\
   East China Normal University,\\
   Shanghai 200241, China\\
   \texttt{52285500019@stu.ecnu.edu.cn} \\
   \And
   Fang Li\thanks{Corresponding author} \\
   School of Mathematical Sciences, \\
   Key Laboratory of MEA \\
   (Ministry of Education),\\
   Shanghai Key Laboratory of PMMP,\\
   East China Normal University,\\
   Shanghai 200241, China\\
   \texttt{fli@math.ecnu.edu.cn} \\
   \And
   Tieyong Zeng \\
   Institute for Advanced Study, \\
   Beijing Normal-Hong Kong Baptist University, \\
   Zhuhai 519000, Guangdong, China\\
    School of Mathematics and Statistics,\\
    Guangzhou Nanfang College,\\ Guangzhou 510970, Guangdong Province, China\\
   \texttt{tieyongzeng@uic.edu.cn} \\
}

\begin{document}

\maketitle

\begin{abstract}
  Plug-and-play (PnP) methods with deep denoisers have shown impressive results in imaging problems. They typically require strong convexity or smoothness of the fidelity term and a (residual) non-expansive denoiser for convergence. These assumptions, however, are violated in Poisson inverse problems, and non-expansiveness can hinder denoising performance. To address these challenges, we propose a cocoercive conservative (CoCo) denoiser, which may be (residual) expansive, leading to improved denoising performance. \rev{Motivated by Helmholtz decomposition, we introduce a training strategy that combines antisymmetric-Jacobian regularization to encourage conservativeness and spectral regularization to encourage cocoercivity.}\footnote{\rev{This revision addresses the \href{https://elvia-atlas.org/atlas/cases/neurips2025-0026-learning-cocoercive-conservative-denoisers-via-helmholtz-decomp.html}{ELVIA post-publication analysis}. Substantive corrections are shown in blue; see Appendix~\ref{appendix:revision}.}} We prove that CoCo denoiser is a proximal operator of a weakly convex function, enabling a restoration model with an implicit weakly convex prior. \rev{Convergence of CoCo-ADMM and CoCo-PEGD to stationary points is established under the stated assumptions.} Extensive experimental results demonstrate that our approach outperforms closely related methods in both visual quality and quantitative metrics. A test code is provided for reproducibility\footnote{ \href{https://github.com/FizzzFizzz/CoCo-PnP}{https://github.com/FizzzFizzz/CoCo-PnP}}.
\end{abstract}

\section{Introduction}

Image restoration is a fundamental task in the computer vision field. Although most works assume that the noise follows Gaussian distribution \cite{liu2012additive,buzzard2018plug}, under low-light condition, such as hyperspectral imaging \cite{priego20174dcaf,dong2018robust}, medical imaging \cite{thanh2015denoising}, and limited photon imaging \cite{chan2016plug}, images are inevitably corrupted by Poisson noises. In this paper, we focus on solving Poisson inverse problems by Plug-and-Play (PnP) methods.

The following formula describes the Poisson corruption procedure:
\begin{equation}\label{model0}
    f\sim  \text{Poisson}(pKu) /p,
\end{equation}
where $u$ is the potential image, $f$ is the noisy image, the peak value $p$ is the average number of photons recieved per pixel, and $K$ is a known linear operator such as identity $\operatorname{I}$, blur kernel, or Radon transform. A small $p$ corresponds to a severe noise environment. In order to recover $u$ from $f$, a variational approach is considered:
\begin{equation}\label{model}
    \arg\min\limits_{u\in V}F(u)+G(u),  G(u) = \lambda\langle \textbf{1},Ku-f\log Ku\rangle, 
\end{equation}
where $V$ is the underlying Hilbert space endowed with the inner product $\langle \cdot, \cdot \rangle$, $F$ denotes the prior regularization term, and the data fidelity $G$ is the Bregman divergence, with $\lambda>0$ being the balancing parameter. Typical choices for $F$ includes total variation \cite{rudin1992nonlinear} and its variants \cite{kindermann2005deblurring,bredies2010total}, weighted nuclear norm \cite{gu2014weighted}, and group-based low rank prior \cite{groupsparsity}. First order methods are employed to find $\hat{u}$, such as the proximal gradient descent method (PGD):
\begin{equation}\label{PGD}
    u^{k+1} =\operatorname{Prox}_{\frac{F}{\beta}}(u^k-\beta^{-1}\nabla G(u^k)),
\end{equation}
and the alternating direction method with multipliers (ADMM) \cite{boyd2011distributed}:
\begin{equation}\label{ADMM}
\begin{array}{ll}
    u^{k+1} &=\operatorname{Prox}_{\frac{F}{\beta}}(v^k-b^k),  \\
    v^{k+1} &=\operatorname{Prox}_{\frac{G}{\beta}}(u^{k+1}+b^k),\\
    b^{k+1} &=b^k+u^{k+1}-v^{k+1},
\end{array}
\end{equation}
where $\beta>0$ is a balancing parameter. For a closed, convex and proper (CCP) function $F: V\rightarrow \bar{\mathbb{R}}=(-\infty,\infty]$, the proximal operator $\operatorname{Prox}_F:V\rightarrow V$ is defined as:
$
    \operatorname{Prox}_F(y) := \arg\min\limits_{x\in V}F(x)+\frac{1}{2}\|x-y\|^2.
$

\textbf{Plug-and-play methods.} By replacing $\operatorname{Prox}_{\frac{F}{\beta}}(\cdot)$ with an off-the-shelf Gaussian denoiser $\operatorname{D}_\sigma(\cdot)$, where $\sigma$ denotes the denoising strength, Venkatakrishnan et al. introduced the plug-and-play ADMM (PnP-ADMM) algorithm \cite{venkatakrishnan2013plug}.
Since then, PnP methods have achieved surprisingly remarkable recovery results in many inverse problems, including bright field electron tomography \cite{sreehari2016plug}, diffraction tomography \cite{sun2019online}, camera image processing \cite{heide2014flexisp}, low-dose CT imaging \cite{bouman2023generative}, Poisson image denoising \cite{le2023preconditioned,hurault2024convergent}, deblurring \cite{laroche2023provably}, inpainting \cite{zhu2023denoising}, and super-resolution \cite{laroche2023deep,kamilov2023plug}.

\textbf{Challenges.} Despite the impressive recovery performance, the convergence analysis of PnP methods remains challenging, especially in the context of Poisson inverse problems. The standard convergence of ADMM and PGD relies on the convexity of both $F$ and $G$. However, a convex prior term $F$ often leads to limited recovery performance. When $F$ is weakly convex, additional assumptions on $G$, such as strong convexity or smoothness, are required. Unfortunately, these assumptions do not hold in Poisson inverse problems. Moreover, in general, a deep denoiser is not a proximal operator, which further complicates the situation.

In order to get a proximal $\operatorname{D}_\sigma$, it should satisfy some Lipschitz properties such as non-expansiveness, and be the (sub)gradient of some CCP potential $\phi$ \cite{moreau1965proximite,gribonval2020characterization}. In the following, we briefly discuss approaches to satisfy these conditions, as well as some related convergent PnP methods.

\textbf{Lipschitz property.} 
There are mainly two methods to ensure the Lipschitz property: \textbf{real spectral normalization (RealSN)} and \textbf{spectral regularization (SR)}. Ryu et al. enforced a contractive $\operatorname{I}-\operatorname{D}_\sigma$ by applying RealSN to the denoiser, which normalized the spectral norm of each layer \cite{ryu2019plug}. However, RealSN is computationally expensive, and was specifically designed for denoisers with cascade residual learning structures, such as DnCNN \cite{zhang2017beyond}, making it unsuitable for other networks like UNet \cite{ronneberger2015u}. 
A more flexible approach is SR technique introduced by Terris et al. \cite{terris2020building}. In order to get a firmly non-expansive $\operatorname{D}_\sigma$, they added a spectral term $\min\{1-\epsilon,\|2\operatorname{J}-\operatorname{I}\|_*\}$ to the original loss function, such that $\|2\operatorname{J}-\operatorname{I}\|_*\le 1$, where 
$\operatorname{J}=\nabla \operatorname{D}_\sigma$ is the Jacobian matrix of $\operatorname{D}_\sigma$, and $\epsilon\in(0,1)$ is a pre-defined parameter. With this spectral constraint, $\operatorname{D}_\sigma$ is firmly non-expansive, and thus becomes the resolvent of some implicitly maximally monotone operator (RMMO), incorporated into the Douglas-Rachford splitting (DRS) method. RMMO-DRS only requires a convex fidelity, making it suitable for Poisson inverse problems. Inspired by SR, Wei et al. proposed to train a strictly pseudo-contractive (SPC) denoiser, termed SPC-DRUNet \cite{AHQS}. By incorporating the Ishikawa process and half-quadratic splitting (HQS), they developed the PnPI-HQS method. The convergence of PnPI-HQS only requires a convex fidelity term $G$, which makes it suitable for Poisson problems. SPC is much weaker than firm or residual non-expansiveness. However, SPC-DRUNet is in general not a proximal operator, and PnPI-HQS does not solve any optimization problem.

\textbf{Conservativeness.} To ensure the conservative property, many works attempt to explicitly define a prior function $\phi$, such that the denoiser is its gradient or proximal operator. Romano et al. proposed the regularization by denoising (RED) \cite{romano2017little}. The RED prior term takes the form of $\phi(x) = \frac{1}{2}\langle x, x-\operatorname{D}_\sigma(x)\rangle$. Under the assumptions that $\operatorname{D}_\sigma$ is locally homogeneous, and that $\nabla \operatorname{D}_\sigma$ is symmetric with spectral radius less than one, Romano et al. proved that $\operatorname{D}_\sigma=\operatorname{I}-\nabla \phi$. Yet the assumptions might be impractical for deep denoisers as reported by Reehorst \& Schniter \cite{reehorst2018regularization}. Instead of training a Gaussian denoiser $\operatorname{D}_\sigma$, Cohen et al. \cite{cohen2021has} parameterized $\phi$ with a neural network, $\operatorname{D}_\sigma=\nabla \phi$. Unfortunately, as verified by Salimans \& Ho \cite{salimans2021should}, and Hurault et al. \cite{hurault2022gradient}, directly modeling $\phi$ with a neural network leads to poor performance. To tackle this, Hurault et al. \cite{hurault2022gradient} introduced the gradient step (GS) denoiser, where $\operatorname{D}_\sigma=\nabla \phi$, with $\phi(x)=\frac{1}{2}\|x\|^2-g_\sigma(x)$, $g_\sigma(x)=\frac{1}{2}\|x-N_\sigma(x)\|^2$, where $N_\sigma$ is a neural network. Then $\operatorname{D}_\sigma=\operatorname{I}-\nabla g_\sigma=N_\sigma+\operatorname{J}_{N_\sigma}^\mathrm{T}(\operatorname{I}-N_\sigma)$, where $\operatorname{J}_{N_\sigma}$ is the Jacobian matrix of $N_\sigma$. $N_\sigma$ uses the light DRUNet architecture \cite{zhang2021plug}, and such $\operatorname{D}_\sigma$ is called GS-DRUNet. Following this, they then proposed Prox-DRUNet, which trains a GS-DRUNet with a non-expansive residual $\operatorname{I}-\operatorname{D}_\sigma $ via SR \cite{hurault2022proximal}. They proved that Prox-DRUNet acts as a proximal operator of some weakly convex prior $F$. The Prox-DRS algorithm is given by plugging $L\operatorname{D}_\sigma+(1-L)\operatorname{I}$ as the denoiser into DRS. Prox-DRS converges for proper closed fidelity $G$ with $L\in[0,0.5)$, making it applicable for Poisson inverse problems. However, this great work still requires a non-expansive residual, which alters the denoising performance, especially for large noises. To apply GS-DRUNet in Poisson inverse problems, Hurault et al. proposed to train a Bregman denoiser to remove Gamma noise \cite{hurault2024convergent}. Two convergent algorithms B-RED and B-PnP are derived. However, experimental results indicate that the Gamma denoiser struggles to remove large Gamma noise, thereby limiting the restoration performance.

\textbf{Motivations.} As discussed above, in order a get a proximal denoiser, it needs to be Lipschitz and conservative.

Typical Lipschitz conditions, such as non-expansiveness and residual non-expansiveness, are restrictive and lead to compromised performance. Weaker assumptions, like (strictly) pseudo-contractive conditions, can not guarantee a proximal denoiser. This limitation motivates us to explore a more suitable assumption that ensures the denoiser remains proximal while maintaining good denoising performance.


For the conservativeness, existing approaches typically construct an explicit potential function $\phi$. While this idea is intuitively appealing and convenient for optimization, it raises a question: Is there any alternative approach to promote the conservativeness, without requiring an explicit prior, by directly regularizing the denoiser, without significantly changing its network structure?

\textbf{Contributions.} To address the issues outlined above, this paper introduces a novel training strategy for learning a proximal denoiser. \rev{Motivated by Helmholtz decomposition, we combine Jacobian-symmetry and spectral penalties to encourage a cocoercive conservative (CoCo) denoiser.} CoCo denoiser proves to be a proximal operator of some implicit non-convex prior function. Cocoerciveness is a weaker constraint on the denoiser, compared to existing constraints like firm or residual non-expansiveness. As a result, the CoCo denoiser not only achieves superior denoising performance but also enhances PnP recovery in Poisson inverse problems. 
Overall, our main contributions are fourfold:

$\bullet$ We introduce a novel assumption, cocoerciveness, for the deep denoisers. Cocoerciveness is strictly weaker assumption than non-epansive type assumptions, and therefore, less restrictive for denoisers.\\
\rev{$\bullet$ Motivated by Helmholtz decomposition, we penalize the antisymmetric part of the Jacobian to encourage conservativeness.} \\
$\bullet$ We prove that a CoCo denoiser is a proximal operator of some implicit prior. An effective training strategy is proposed to promote these properties.\\
$\bullet$ A Poisson inverse model with an implicit prior is derived. \rev{Convergence of PnP methods with CoCo denoisers is established under the stated assumptions.}

\section{CoCo denoisers}

In this section, we propose CoCo denoisers. First, we introduce the cocoercive assumption, and its spectral distribution on the complex plane. \rev{Next, motivated by Helmholtz decomposition, we use Jacobian symmetry to characterize conservativeness and motivate its regularization.} Then, we prove that \rev{an averaged CoCo denoiser $\operatorname{D}_\sigma^t$, $0\le t<1$,} is the proximal operator of some weakly convex and smooth prior function $F:V\rightarrow \bar{\mathbb{R}}:=\mathbb{R}\cup\{\infty\}$, \rev{$\operatorname{D}_\sigma^t = \operatorname{Prox}_{\frac{F}{\beta}}$}. Finally, a restoration model with an implicit weakly convex prior is given.

\subsection{Cocoercive denoisers}
Let $V=\mathbb{R}^n$ be a real Hilbert space\footnote{\rev{The convergence results here are stated for finite-dimensional $V=\mathbb{R}^n$.} The spectrum is the eigenvalue set, and $\|\cdot\|_*$ denotes the spectral (operator) norm.} with inner product $\langle \cdot, \cdot \rangle$, and the induced norm $\|x\|=\sqrt{\langle x,x\rangle}$. Let $\operatorname{D}:V\rightarrow V$ be an operator\footnote{An operator in general may not be single-valued. However, since the operator $\operatorname{D}$ in this paper is a denoiser, its output is unique given fixed input. Thus, we only consider the single-valued operator here.}. An operator $\operatorname{D}:V\rightarrow V$ is said to be $\gamma$-cocoercive ($\gamma\in[0,\infty)$), if $\forall x,y\in V$, there holds:
\begin{equation}\label{def coco}
    \langle x-y, \operatorname{D}(x)-\operatorname{D}(y)\rangle\ge \gamma \|\operatorname{D}(x)-\operatorname{D}(y)\|^2.
\end{equation}
Cocoercive operators are an important class of monotone operators. Many operators are cocoercive operators. Below are two typical cocoercive operators:\\
    $\bullet$ $\operatorname{D}$ is firmly non-expansive: ($\forall x,y\in V$)
    \begin{equation}
        \langle x-y, \operatorname{D}(x)-\operatorname{D}(y)\rangle \ge \|\operatorname{D}(x)-\operatorname{D}(y)\|^2.
    \end{equation}
    Such $\operatorname{D}$ is $1$-cocoercive.\\
    $\bullet$ $\operatorname{D}$ is residual non-expansive: ($\forall x,y\in V$)
    \begin{equation}
        \|(\operatorname{I}-\operatorname{D})\circ(x)-(\operatorname{I}-\operatorname{D})\circ(y)\|\le \|x-y\|.
    \end{equation}
    Such $\operatorname{D}$ is $0.5$-cocoercive.

When $\gamma>0$, cocoercive assumption make denoisers Lipschitz: by Cauchy-Schwarz inequality, 
\begin{equation}
    \langle x-y, \operatorname{D}(x)-\operatorname{D}(y)\rangle\le \|x-y\|\|\operatorname{D}(x)-\operatorname{D}(y)\|.
\end{equation}
Therefore, by (\ref{def coco}), $\|\operatorname{D}(x)-\operatorname{D}(y)\|\le \frac{1}{\gamma}\|x-y\|$, that is, $\operatorname{D}$ is $\frac{1}{\gamma}$-Lipschitz. When $\gamma<0.5$, both the operator and its residual part are expansive. In general, a smaller $\gamma$ corresponds to less contraint, and therefore a more powerful denoiser.

\begin{figure}[h]
\setlength{\fboxsep}{3pt} 
\setlength{\fboxrule}{0.6pt} 
\begin{minipage}{0.38\linewidth}
  \centerline{\includegraphics[width=5.5cm]{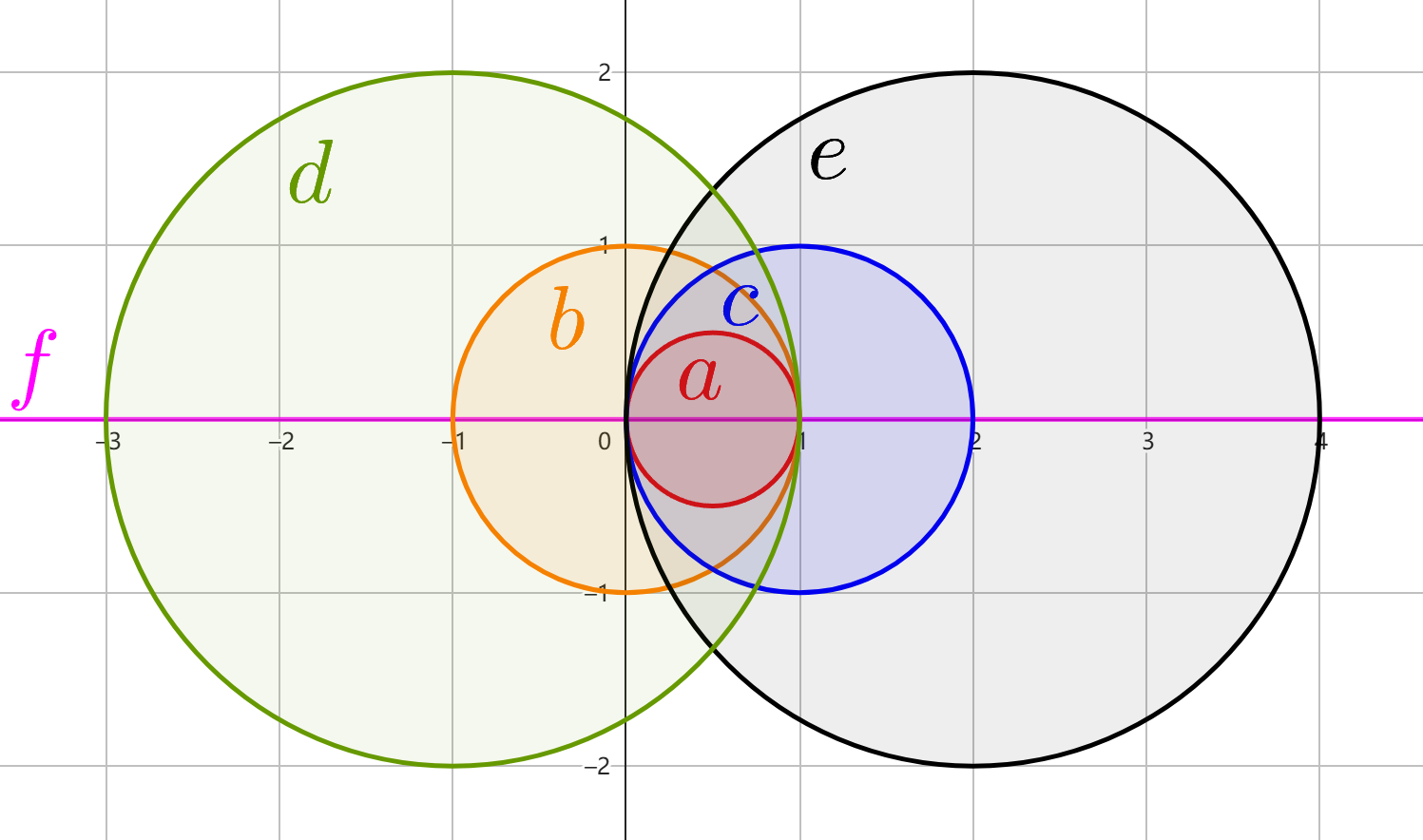}}
\end{minipage}
\hfill
\begin{minipage}{0.20\linewidth}
  \centerline{ \fbox{\includegraphics[width=2.51cm, height=2.51cm]{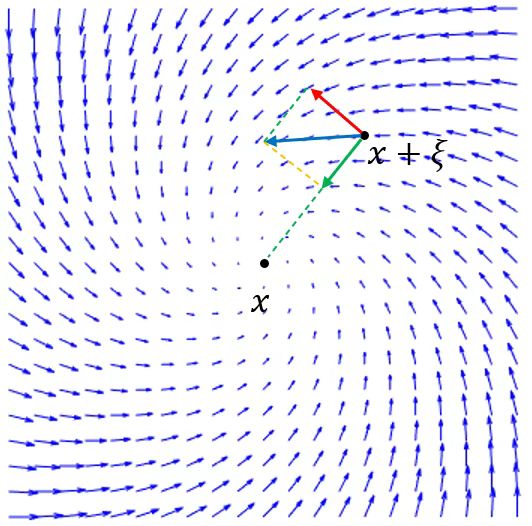}}}
  \centerline{(i)}
\end{minipage}
\hfill
\begin{minipage}{0.20\linewidth}
  \centerline{\fbox{\includegraphics[width=2.51cm, height=2.51cm]{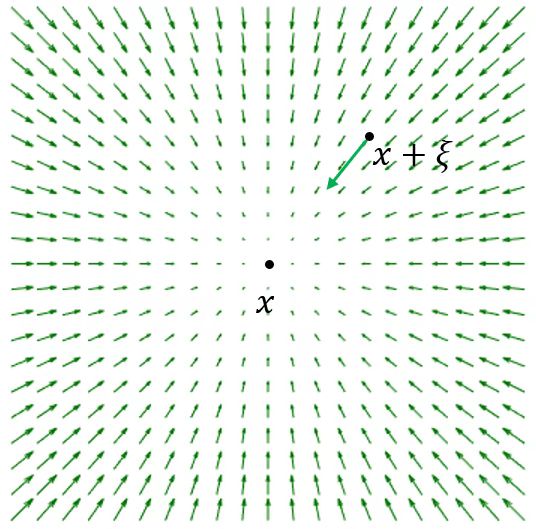}}}
  \centerline{(ii)}
\end{minipage}
\hfill
\begin{minipage}{0.20\linewidth}
  \centerline{\fbox{\includegraphics[width=2.51cm, height=2.51cm]{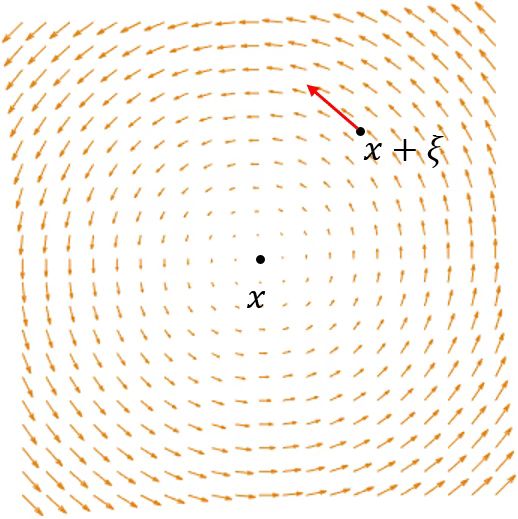}}}
  \centerline{(iii)}
\end{minipage}
\centering\caption{Left: Spectrum distributions of the Fr\'{e}chet differential (Jacobian matrix) on the complex plane under different assumptions. (a) Firmly non-expansiveness, $\operatorname{Sp}(\operatorname{J})\subset\{z\in\mathbb{C}: |2z-1|\le 1\}$; (b) Non-expansiveness, $\operatorname{Sp}(\operatorname{J})\subset\{z\in\mathbb{C}: |z|\le 1\}$; (c) Residual non-expansiveness,  $\operatorname{Sp}(\operatorname{J})\subset\{z\in\mathbb{C}:|z-1|\le1 \}$; (d) $\frac{1}{2}$-strictly pseudo-contractiveness, $\operatorname{Sp}(\operatorname{J})\subset\{z\in\mathbb{C}:|z+1|\le 2\}$; (e) $0.25$-cocoerciveness, $\operatorname{Sp}(\operatorname{J})\subset\{z\in\mathbb{C}: |0.5z-1|\le1\}$; (f) Conservativeness, $\operatorname{Sp}(\operatorname{J})\subset\mathbb{R}$. In general, a larger region means less restrictive assumption. The spectrum of $\gamma$-CoCo denoisers ($\gamma=0.25$) lies inside the interval $(\text{f})\cap(\text{e})=[0,4]$.  \\
Right: \rev{A two-dimensional schematic illustration of denoising, conservative, and rotational fields.} $x$ denotes the clean image point, $\xi$ denotes the Gaussian noise, and $x+\xi$ denotes the noisy image. The arrow ``$\rightarrow$" represents the denoising direction. (i) Denoising field $\operatorname{D}$. (ii) Conservative field $\operatorname{D}_c$. (iii) Hamiltonian field $\operatorname{D}_h$.}
\label{fig field}
\end{figure}

\textbf{Spectral analysis on $\operatorname{D}$.} \rev{For $\gamma>0$,} note that (\ref{def coco}) can be equivalently transformed into the following inequality: ($\forall x,y\in V$)
\begin{equation}\label{def2 coco}
    \|(2\gamma \operatorname{D}-\operatorname{I})\circ(x)-(2\gamma \operatorname{D}-\operatorname{I})\circ(y)\|\le \|x-y\|.
\end{equation}
Based on the mean value theorem \cite{duistermaat2004multidimensional}, (\ref{def2 coco}) is transformed into:
\begin{equation}\label{def3 coco}
    \|2\gamma \operatorname{J}(x)-\operatorname{I}\|_*\le 1, \forall x\in V,
\end{equation}
where $\operatorname{J}(x)$ is the Fr\'echet differential, that is the Jacobian matrix when $V$ is finite dimensional, of $\operatorname{D}$ at the point $x$, and $\|\cdot\|_*$ denotes the spectral norm. Since the spectral radius $\rho(\cdot)$ is always no larger than the spectral norm $\|\cdot\|_*$, we have that:
\begin{equation}\label{def4 coco}
    \rho(2\gamma \operatorname{J}(x)-\operatorname{I})\le 1, \forall x\in V.
\end{equation}
Let $\operatorname{Sp}(\cdot)$ denote the spectrum set, that is, the eigenvalue set when $V$ is finte dimensional:
\begin{equation}
    \operatorname{Sp}(\operatorname{J}) := \{z\in\mathbb{C}: z\operatorname{I}-\operatorname{J} \text{ is not invertible} \}.
\end{equation}
(\ref{def4 coco}) implies $\operatorname{Sp}(\operatorname{J}(x))\subseteq \{z\in \mathbb{C}: \ |2\gamma z-1|\le 1 \}$. 

Lemma \ref{lemma 1} gives an \textit{equivalent} condition of cocoerciveness in (\ref{def coco})-(\ref{def2 coco}), and thereby enables the training through SR.
\begin{lemma}[Proof in Appendix \ref{Appendix lemma1}]\label{lemma 1}
Let \rev{$\operatorname{D}\in\mathcal{C}^1(V;V)$}, and $\operatorname{J}=\nabla \operatorname{D}$ be its Fr\'echet differential. Then $\operatorname{D}$ is $\gamma$-cocoercive (\rev{$\gamma>0$}), if and only if $\|2\gamma\operatorname{J}(x)-\operatorname{I}\|_*\le1$ for any $x\in V$.
\end{lemma}
Since the spectral radius is no larger than the spectral norm, we are able to plot the spectral distribution on the complex plane $\mathbb{C}$ in Fig. \ref{fig field}. \rev{A larger region allows greater flexibility in the Jacobian spectrum.} Figs. \ref{fig field} (a), (c), (e) show that firm non-expansiveness, and residual non-expansiveness are special cases of cocoerciveness. \rev{For smaller $\gamma$, cocoercivity is a less restrictive assumption for deep denoisers.}

\subsection{Conservative denoisers}

We consider the conservative assumption intuitively. Let $\operatorname{D}\in\mathcal{C}^1[V]$ be a denoiser. $\operatorname{D}$ is said to be conservative, if there is a potential function $\phi:V\rightarrow \bar{\mathbb{R}}$, such that $\operatorname{D}= \nabla \phi$. Geometrically, $\operatorname{D}$ is mapping from $V$ to $V$, thus can be viewed as a vector field. 

\begin{revision}
Under suitable domain, boundary, and regularity assumptions, Helmholtz--Hodge decomposition separates a vector field into gradient and divergence-free parts \cite{bhatia2012helmholtz}. It motivates the distinction illustrated in Fig.~\ref{fig field}. The pointwise Jacobian splitting
\[
\operatorname{J}=\operatorname{S}+\operatorname{A},\qquad
\operatorname{S}=\tfrac12(\operatorname{J}+\operatorname{J}^\top),\quad
\operatorname{A}=\tfrac12(\operatorname{J}-\operatorname{J}^\top),
\]
is an algebraic decomposition \cite{balduzzi2018mechanics} measuring local Jacobian symmetry; its two parts need not be Jacobians of global Helmholtz components.

For $\operatorname{D}\in\mathcal{C}^1(\mathbb{R}^n;\mathbb{R}^n)$, conservativeness is equivalent to $\operatorname{J}(x)=\operatorname{J}(x)^\top$ for every $x$ \cite{galbis2012vector}. We therefore penalize $\|\operatorname{J}-\operatorname{J}^\top\|_*$ as a local measure of departure from conservativeness. The term \emph{Hamiltonian regularization} below refers to this antisymmetric-Jacobian penalty.
\end{revision}

\subsection{Characterization on CoCo denoisers}

In the context of PnP, the denoiser $\operatorname{D}_\sigma$ is expected to be proximal. Building upon prior findings by \cite{gribonval2020characterization} (see Appendix \ref{appendix characterization}), we present the following Theorem \ref{thm 1}. A more general characterization on CoCo denoisers is given in Theorem \ref{thm 0} in Appendix \ref{appendix general characterization}.

\begin{theorem}[Proof in Appendix \ref{Appendix thm1}]\label{thm 1}
Let $\operatorname{D}_\sigma\in \mathcal{C}^1[V]$. $\beta=\frac{1}{\sigma^2}$. $\operatorname{D}_\sigma$ satisfies that:\\
$\bullet$ $\operatorname{D}_\sigma$ is conservative;\\
$\bullet$ $\operatorname{D}_\sigma$ is $\gamma$-cocoercive with $\gamma\in(0,1)$.\\
Let $\operatorname{D}_\sigma^t=t\operatorname{D}_\sigma+(1-t)\operatorname{I}$, $t\in\left[0,1\right)$. It holds that:\\
$\bullet$ there exists a $r$-weakly convex \rev{$\mathcal{C}^2$ function $F:V\rightarrow\mathbb{R}$}, $r(t)=\beta\frac{t-\gamma t}{t+\gamma -\gamma t}$, such that $\operatorname{D}^t_\sigma(x)\in \operatorname{Prox}_{\frac{F}{\beta}}(x), \forall x\in V$;\\
$\bullet$ $\partial F$ is $L$-Lipschitz, where 
\begin{equation}
    L(t)=\left\{
    \begin{array}{ll}
        \beta\frac{t}{1-t}\ge r(t), & \text{if } t\ge  \frac{1-2\gamma}{2-2\gamma}\text{ and } t\in[0,1) \text{ (Case 1)};\\
       r(t)=\beta\frac{t-\gamma t}{t+\gamma -\gamma t}, & \text{if } t\le \frac{1-2\gamma}{2-2\gamma} \text{ and } t\in[0,1) \text{ (Case 2)}. 
    \end{array}\right.
\end{equation}
\end{theorem}
\begin{remark}
$\operatorname{D}_\sigma^t$ is an averaged version of $\operatorname{D}_\sigma$. As reported by \cite{hurault2022proximal}, when $\operatorname{D}_\sigma$ is residual non-expansive, PnP with $\operatorname{D}_\sigma^t$ out-performs PnP with $\operatorname{D}_\sigma$ in many inverse problems. 
\end{remark}
\begin{remark}
If $\gamma=0.5$, $\operatorname{D}_\sigma$ is residual non-expansive, $\operatorname{D}_\sigma^{t}$ is then residual $t$-Lipschitz, and $\beta^{-1}F$ is $\frac{t}{1+t}$-weakly convex, with $\beta^{-1}\partial F$ being $\frac{t}{1-t}$-Lipschitz. This special case recovers the result by \cite{hurault2022proximal}. 
\end{remark}

\begin{remark}
By Theorem \ref{thm 1}, we know that $\operatorname{D}_\sigma^t$ is a proximal operator of a weakly convex function $\frac{F}{\beta}$. We arrive at the following implicit non-convex restoration model:
\begin{equation}\label{our model}
    \hat{u}\in \arg\min_{u\in V} F(u)+G(u;f),\ 
    G(u;f) = \lambda\langle \textbf{1},Ku-f\log Ku\rangle.
\end{equation}
\end{remark}
\vspace{-4ex}

\section{Training strategy}
Let $\theta$ be the parameter weights in $\operatorname{D}_\sigma$ to be optimized, $\pi$ be the distribution of the training set of clean images, and $[\sigma_{\min},\sigma_{\max}]$ be the interval of the noise level. 

Based on Lemma \ref{lemma 1} and the pioneer works by \cite{terris2020building,pesquet2021learning}, we encourage the cocoerciveness by the SR technique. The spectral regularization term $L_s$ takes the form of:
\begin{equation}\label{loss spectral}
    L_s(\theta) = \mathbb{E}_{x,\sigma,\xi}\max\{\|2\gamma\operatorname{J}-\operatorname{I}\|_*,1-\epsilon\},
\end{equation}
where $\theta$ is the network weights, $\epsilon\in(0,1)$ is a parameter that controls the constraint. $x\sim \pi,\sigma\sim U[\sigma_{\min},\sigma_{\max}],\xi\sim \mathcal{N}(0,\sigma^2\operatorname{I})$. $\operatorname{J} = \operatorname{J}(x+\xi)=\nabla \operatorname{D}_\sigma(x+\xi;\theta)$.

To encourage a conservative denoiser, we propose the Hamiltonian regularization term $L_h$:
\begin{equation}\label{loss hamilton}
    L_h(\theta) = \mathbb{E}_{x,\sigma,\xi}\|\operatorname{J}-\operatorname{J}^\top\|_*.
\end{equation}

Combining (\ref{loss hamilton}) and (\ref{loss spectral}), the overall loss function of $\operatorname{D}_\sigma$ is 
\begin{equation}\label{loss}
    Loss(\theta)=\mathbb{E}\|\operatorname{D}_\sigma(x+\xi;\theta)-x\|_1+\alpha_1 L_h +\alpha_2 L_s.
\end{equation}
$\alpha_1,\alpha_2>0$ are balancing parameters. The first term in (\ref{loss}) ensures that $\operatorname{D}_\sigma$ is a Gaussian denoiser, \rev{the second term encourages conservativeness, and the third encourages $\gamma$-cocoercivity.} Detailed calculation of $\|(2\gamma\operatorname{J}-\operatorname{I})\|_*$ and $\|\operatorname{J}-\operatorname{J}^\mathrm{T}\|_*$ is given in Appendix \ref{appendix training}.

\section{PnP methods with CoCo denoisers}
In order to solve the implicit nonconvex restoration model in (\ref{our model}), we first consider PnP-ADMM with CoCo denoisers, termed CoCo-ADMM. Then, we replace the fidelity term $G$ in (\ref{our model}) by its Moreau envelope, and derive the proximal envelope gradient descent method (PEGD) with CoCo denoisers (CoCo-PEGD). The convergence are provided.

\textbf{CoCo-ADMM.}
We first consider CoCo-ADMM iterations:
\begin{equation}\label{PnP-ADMM}
\begin{array}{ll}
u^{k+1}&=\operatorname{Prox}_{\frac{G}{\beta}}(v^k-b^k), \vspace{0.5ex}\\
v^{k+1}&=\operatorname{D}^t_\sigma(u^{k+1}+b^k),\\
b^{k+1}&=b^k+u^{k+1}-v^{k+1},
\end{array}
\end{equation}
where $\operatorname{D}^t_\sigma$ is defined in Theorem \ref{thm 1}, and $\beta=\frac{1}{\sigma^2}$. The PnP-ADMM algorithm in (\ref{PnP-ADMM}) with a $\gamma$-CoCo denoiser is referred to as $\gamma$-CoCo-ADMM, or CoCo-ADMM for short.

When the denoiser $\operatorname{D}_\sigma\in\mathcal{C}^1[V]$ is a CoCo denoiser satisfying the conditions in Theorem \ref{thm 1}, \rev{and the assumptions below hold \cite{attouch2010proximal,frankel2015splitting}}, the global convergence of PnP-ADMM in (\ref{PnP-ADMM}) can be established as follows.
\begin{theorem}[Proof in Appendix \ref{appendix main thm}]\label{main thm}
Let $F:V\rightarrow \bar{\mathbb{R}}$ be a coercive weakly convex function in Theorem \ref{thm 1} such that $\operatorname{D}_\sigma^t\in\operatorname{Prox}_{\frac{F}{\beta}}$. \rev{$G:V\rightarrow\bar{\mathbb{R}}$ is proper, lower semicontinuous, convex, and bounded below. The joint augmented Lagrangian $L_\beta$ in \eqref{envelope} is assumed to be KL.} {$\gamma\in(0,1)$. \rev{$0<t<\min\{t_0,\gamma/(1-\gamma)\}$}}, where $t_0=t_0(\gamma)$ is the positive root of the equation 
\begin{equation}\label{eq t gamma}
(2-2\gamma)t^3+\gamma t^2+2\gamma t-\gamma=0.
\end{equation}
\rev{For $u^0\in\operatorname{dom}G$ and finite $v^0,b^0$,} the sequence $\{(u^k,v^k,b^k)\}$ generated by (\ref{PnP-ADMM}) converges globally to a point $(u^*,v^*,b^*)$, and that $u^*=v^*$ is a stationary point of the model (\ref{our model}).
\end{theorem}

\begin{remark}The KL property has been widely used to study the convergence of optimization algorithms in the nonconvex and nonsmooth setting \cite{attouch2010proximal}. Many functions, in particular all the real semi-algebraic functions, satisfy this property. \rev{The Poisson fidelity is bounded below for nonnegative observations, with its usual extended-domain convention.}
\end{remark}
\begin{remark}
We consider several special cases of Theorem \ref{main thm}. Let $\operatorname{D}_\sigma$ be conservative. When $\gamma=0.5$, $\operatorname{D}_\sigma$ is residual non-expansive, and $t_0\approx0.3761$. This recovers the case by \cite{hurault2024convergent}; when $\gamma=0.25$, both $\operatorname{D}_\sigma$ and its residual $\operatorname{I}-\operatorname{D}_\sigma$ can be expansive, but $0.5\operatorname{D}_\sigma-\operatorname{I}$ is non-expansive. In this case, $t_0\approx 0.3333$. When $\gamma=1$, $\operatorname{D}_\sigma$ is firmly non-expansive and is a proximal operator of some CCP function by Moreau's theorem. In this case, $t_0$ is no more needed to ensure the convergence, \rev{under the usual solution-existence assumptions for convex ADMM.} In experiments, we use $\operatorname{D}_\sigma(u^{K}+b^{K})$ as the final output.
\end{remark}

\textbf{CoCo-PEGD.}
Now we consider the PGD algorithm with a CoCo denoiser. The standard PGD takes the form of:
\begin{equation}\label{PGD1}
    u^{k+1} =\operatorname{Prox}_{\frac{F}{\beta}}\left(u^k-\beta^{-1}\nabla G(u^k)\right).
\end{equation}
When $F$ is $r$-weakly convex, $\nabla G$ is $L_G$-Lipschitz, and $\beta^{-1}<\max\{\frac{2}{L_G+r},\frac{1}{L_G}\}$, \rev{the iteration converges under the boundedness and KL assumptions in \cite{attouch2013convergence,hurault2024convergent}.}

However, in the Poisson inverse problems, $G$ is not smooth, because $\nabla G$ is not Lipschitz. In order to apply the PGD algorithm, we smooth the fidelity $G$ in (\ref{our model}) and (\ref{PGD1}) by replacing it with its Moreau envelope $^\alpha {G}$ ($\alpha\in(0,\infty)$):
$
    ^\alpha{G}(u):=\min_v G(v)+\frac{1}{2\alpha}\|v-u\|^2.
$

Since $G$ is CCP, $^\alpha G$ is differentiable \cite{bauschke2017correction}:
$
    \nabla ^\alpha {G}(u)=\frac{1}{\alpha}(\operatorname{I}-\operatorname{Prox}_{\alpha G}),
$ and that $\nabla ^\alpha {G}$ is $\frac{1}{\alpha}$-Lipschitz. Since there are already a step parameter $\beta$ and a balancing parameter $\lambda$ in $G$, throughout this paper, we set $\alpha=1$. Now the iteration becomes:
\begin{equation}\label{PEGD}
    u^{k+1} =\operatorname{Prox}_{\frac{F}{\beta}}(u^k-\beta^{-1}\nabla {^1G}(u^k)).
\end{equation}
In (\ref{PEGD}), $u$ alternates between a proximal operator $\operatorname{Prox}_{\frac{F}{\beta}}$, and a gradient descent step on the envelope of $G$. (\ref{PEGD}) is referred to as the proximal envelope gradient descent method (PEGD). Similar to CoCo-ADMM (\ref{PnP-ADMM}), we replace $\operatorname{Prox}_{\frac{F}{\beta}}$ with $\operatorname{D}_\sigma^t$, and arrive at CoCo-PEGD:
\begin{equation}\label{CoCo PEGD}
     u^{k+1} =\operatorname{D}_\sigma^t(u^k-\beta^{-1}\nabla {^1G}(u^k)).
\end{equation}
By Theorem \ref{thm 1}, $\operatorname{D}_\sigma^t=\operatorname{Prox}_{\frac{F}{\beta}}$, with $F$ being $r$-weakly convex. Then by the Theorem 1 in \cite{hurault2024convergent}, \rev{CoCo-PEGD converges under the assumptions below when $\beta^{-1}<\max\{\frac{2}{1+r},1\}$.} This convergence result is summarized in Theorem \ref{main thm 3}. \rev{In particular, $\beta>1$ satisfies the step-size condition.}

\begin{theorem}[Proof in Appendix \ref{Appendix thm3}]\label{main thm 3}
Let $F:V\rightarrow \bar{\mathbb{R}}$ be \rev{a proper lower semicontinuous, coercive $r$-weakly convex prior constructed as in Appendix~\ref{appendix general characterization}, with $r=\beta t(1-\gamma)/(\gamma+t(1-\gamma))$} such that $\operatorname{D}_\sigma^t=\operatorname{Prox}_{\frac{F}{\beta}}$. \rev{$G:V\rightarrow\bar{\mathbb{R}}$ is CCP and bounded below, and $F+{^1G}$ is assumed to be KL.} $\gamma\in[0.25,1]$. $t\in(0,1]$. $0<\beta^{-1}<\max\{\frac{2}{1+r},1\}$. \rev{For $u^0\in\operatorname{dom}F$,} the sequence $\{u^k\}$ generated in (\ref{CoCo PEGD}) converges to a stationary point of:
\begin{equation}\label{our model PEGD}
    \hat{u}\in \arg\min_{u\in V} F(u)+ {^1G}(u;f). 
\end{equation}
\end{theorem}\vspace{-3ex}

\section{Experiments}
\vspace{-1ex}

All the experiments are conducted under Linux system, Python 3.8.12 and Pytorch 1.10.2 with a RTX 3090 GPU.

\textbf{Training details.} For $\operatorname{D}_\sigma$, we select DRUNet \cite{zhang2021plug}, which combines a residual learning \cite{he2016deep} and UNet architecture \cite{ronneberger2015u}. DRUNet takes both the noisy image and the noise level $\sigma$ as input, making it convenient for PnP image restoration.

To train a CoCo denoiser, we collect 800 images from the DIV2K dataset \cite{Ignatov_2018_ECCV_Workshops} as the training set and used a batch size of 32 and patch size $128$. We add Gaussian noise with randomly generated standard deviation values in the range of $[\sigma_{\min},\sigma_{\max}]=[0,50]$ to the clean image. Adam optimizer is applied to train the model with learning rate $lr=10^{-4}$. We set $\alpha_1 = 1, \alpha_2=0.01$, and $\epsilon=0.1$ \rev{to encourage the regularity conditions}. \rev{To estimate the spectral norms} $\|\operatorname{J}(x)-\operatorname{J}^\top(x)\|_*$ and $\|(2\gamma\operatorname{J}-\operatorname{I})(x)\|_*$, we use the power iterative method \cite{golub2013matrix} with $30$ iterations. \par

\textbf{Denoising performance.} We evaluate the Gaussian denoising performances of the proposed denoiser CoCo-DRUNet, strictly pseudo-contractive DRUNet (SPC-DRUNet) \cite{AHQS}, resolvent of a maximally monotone operator (RMMO) \cite{pesquet2021learning} which is firmly non-expansive, GS-DRUNet  \cite{hurault2022gradient}, Prox-DRUNet \cite{hurault2022proximal}, the standard DRUNet, FFDNet \cite{zhang2018ffdnet}, and DnCNN \cite{zhang2017beyond}.

The PSNR values are given in Table \ref{denoising performance}. A $\gamma$-cocoercive conservative denoiser is referred to as $\gamma$-CoCo-DRUNet. We see in Table \ref{denoising performance} that compared with DnCNN, FFDNet, and the regularized denoisers, $0.25$-CoCo-DRUNet has competitive denoising performance. This is because that: cocoercive and conservative properties are less restrictive for deep denoisers; denoiser is trained with a different loss function. \rev{Both $0.50$-CoCo-DRUNet and Prox-DRUNet target conservativeness and residual non-expansiveness, and show similar denoising performance.}

\begin{table}[htbp]
\caption{Left: Average denoising PSNR performance of different denoisers on CBSD68, for various noise levels $\sigma$; Right: Mean symmetry error $\|\operatorname{J}-\operatorname{J}^\top\|_* $ ($N=1$) and maximal \rev{estimates} of the norm $\|2\gamma\operatorname{J}-\operatorname{I}\|_*$ ($N=30$) for various noise levels $\sigma$ and $\gamma=0.50,0.25$.}\label{denoising performance}
\centering
\begin{minipage}{0.40\linewidth}
\resizebox{1\textwidth}{!}{
\begin{tabular}{cccc}
\toprule 
        &$\sigma=15$ &$\sigma=25$ &$\sigma=40$ \\
\midrule
FFDNet                               &33.86&31.18&28.81 \\
DnCNN                       &33.88 & 31.20 & 28.89          \\
DRUNet&\textbf{34.14} & \textbf{31.54}& \textbf{29.33}   \\   
\midrule
RMMO   &32.21 & 29.99 & 27.87         \\
GS-DRUNet   & 33.56&31.01 &28.81 \\
Prox-DRUNet   & 33.18& 30.60 & 28.38     \\
SPC-DRUNet & 33.90 & 31.29 & {29.10}    \\
$0.50$-CoCo-DRUNet&{33.38} & {30.65}& {28.25} \\
$0.25$-CoCo-DRUNet& \textbf{34.00}& \textbf{31.38}&\textbf{29.16}   \\
\bottomrule
\end{tabular}
}
\end{minipage}
\hfill
\begin{minipage}{0.55\linewidth}
\resizebox{1\textwidth}{!}{
\begin{tabular}{ccccc}
\toprule
             & 15     & 25     & 40     &  Norms  \\
\midrule
DRUNet              &  $4.1$e+0& $4.2$e+1 & $1.8$e+1 & $\|\operatorname{J}-\operatorname{J}^\top\|_* $  \\
$0.50$-CoCo-DRUNet  & $8.6$e-3 & $2.5$e-4 & $7.1$e-4 & $\|\operatorname{J}-\operatorname{J}^\top\|_* $   \\
$0.25$-CoCo-DRUNet  & $3.1$e-4 & $1.8$e-4 & $3.9$e-4 & $\|\operatorname{J}-\operatorname{J}^\top\|_* $   \\

\midrule
DRUNet              & $3.285$ &$4.343$  & $6.283$ & $\|\operatorname{J}-\operatorname{I}\|_*$ \\
$0.50$-CoCo-DRUNet  & $0.994$ & $0.992$ &$0.972$  &  $\|\operatorname{J}-\operatorname{I}\|_*$ \\
$0.25$-CoCo-DRUNet  & $0.986$ & $0.969$ &$0.982$  &  $\|0.5\operatorname{J}-\operatorname{I}\|_*$ \\
\bottomrule
\end{tabular}
}
\end{minipage}
\end{table}

\textbf{Assumption validations.}
In the experiments, the symmetric Jacobian and non-expansive residual are softly constrained by the loss function (\ref{loss}) with trade-off parameters $\alpha_1,\alpha_2$. \rev{We report empirical diagnostics in Table~\ref{denoising performance}.} We use the symmetry error $\|\operatorname{J}(x)-\operatorname{J}^\top(x)\|_* $ over $100$ different patches with $N=1$ as in Algorithm \ref{alg power}. A smaller error denotes a higher Jacobian symmetry. \rev{We also estimate $\|2\gamma\operatorname{J}(x)-\operatorname{I}\|_*$ to assess cocoercivity on the sampled inputs.}

As shown in Table \ref{denoising performance}, DRUNet without regularization terms has an expansive residual part. \rev{The regularized CoCo-DRUNet has smaller sampled norm estimates.} Besides, the proposed Hamiltonian regularization term indeed encourages a symmetric Jacobian: in Table \ref{denoising performance}, we see that Hamiltonian regularization reduces the symmetry error. It validates the effectiveness of the proposed training strategy. Ablation study on the parameters $\alpha_1$ and $\alpha_2$ in \eqref{loss} are provided on Appendix \ref{sec:ablation}.

\textbf{PnP restoration.} We apply CoCo-ADMM and CoCo-PEGD to multiple Poisson inverse problems including: \textbf{photon limited deconvolution}, \textbf{single photon imaging} in a real-world low-light setting in Appendix \ref{appendix single}, \textbf{low-dose CT reconstruction} tasks in Appendix \ref{appendix CT}, \textbf{Poisson denoising} in Appendix \ref{appendix poisson denoising}. \textbf{Computational time}, \textbf{performances under extreme conditions}, and \textbf{blind deblur and denoise results}, are reported in Appendices \ref{sec:cputime}, \ref{sec:extreme}, \ref{sec:deblur}, and \ref{sec:denoise} respectively.

When $K$ in (\ref{model}) is not the identity matrix, in general, $\operatorname{Prox}_G$ has no closed-form solution. Since $G$ is convex, we use ADMM to solve $\operatorname{Prox}_G$ efficiently. For details, please refer to Appendix \ref{appendix proximal}.

We choose $\gamma=0.25$ since $0.25$-CoCo-DRUNet has a satisfactory performance. \rev{For $\gamma=0.25$, the corrected CoCo-ADMM condition is $0<t<1/3$. CoCo-PEGD has the separate step-size condition in Theorem~\ref{main thm 3}. } For both CoCo-ADMM and PEGD, we need to calculate the proximal operator $\operatorname{Prox}_{\frac{G}{\beta}}$. We detail the calculations for each task in Appendix \ref{appendix proximal}. For each method, we fine tune the parameters to achieve the best quantitive PSNR values. The proposed methods are initialized with the observed image, that is $u^0=v^0=f, b^0=0$. 

\begin{table}[htbp]
\caption{Left: Average deconvolution PSNR and SSIM performance by different methods on CBSD68 with Levin's $8$ kernels and Poisson noises with peak value $p=50$ and $p=100$. Right: Average low-dose CT reconstruction PSNR and SSIM performance by different methods on Mayo's dataset with Poisson noises with peak value $p=500$, and $p=100$.}\label{tab deblur}
\begin{center}
\begin{minipage}{0.45\linewidth}
\resizebox{1.0\textwidth}{!}{
\begin{tabular}{ccccc}
\toprule
& \multicolumn{2}{c}{$p=100$} & \multicolumn{2}{c}{$p=50$} \\
\midrule
 & PSNR & SSIM & PSNR & SSIM \\
\midrule
DPIR        & 26.51 & \textbf{0.7419} & 25.38 & 0.6910  \\
RMMO-DRS         & 25.94 & 0.7019 & 25.10 & 0.6546 \\
Prox-DRS & 25.68 & 0.6764& 25.21&0.6572\\
DPS & 23.65 & 0.6062 & 23.10 & 0.5816\\
DiffPIR & 24.82 & 0.6429 & 24.08 & 0.6074\\
SNORE & 26.33 & 0.7158 & 25.21 & 0.6719\\
PnPI-HQS     & 26.42 & 0.7158 & 25.61 & {0.6956}  \\ 
B-RED       & 24.09 & 0.6794 & 23.78 & 0.6603  \\
CoCo-ADMM & \textbf{26.89} & \underline{0.7358} & \textbf{26.00} & \textbf{0.7026}\\
CoCo-PEGD & \underline{26.79} & 0.7323 & \underline{25.90} & \underline{0.6958} \\
\bottomrule
\end{tabular}
}
\end{minipage}
\hfill
\begin{minipage}{0.45\linewidth}
\resizebox{1.0\textwidth}{!}{
\begin{tabular}{ccccc}
\toprule
& \multicolumn{2}{c}{$p=500$}  & \multicolumn{2}{c}{$p=100$}\\
\midrule
 & PSNR & SSIM &  PSNR & SSIM\\
\midrule 
FBP  & 28.76 & 0.5212&  24.10 & 0.2974\\
PWLS-TGV & 33.16&0.8461& 30.43 & 0.7750   \\
PWLS-CSCGR  & 35.45 &0.8909 & 33.78&0.8534\\
UNet & 36.93 & 0.9139  & 35.19 & 0.8875\\
WNet & {37.12} & 0.9266 & 35.98 &0.9130\\
CoCo-ADMM & \underline{37.63} & \textbf{0.9403} & \textbf{36.68} & \textbf{0.9194} \\
CoCo-PEGD  & \textbf{37.72} & \underline{0.9391} &  \underline{36.43} & \underline{0.9159}\\
\bottomrule
\end{tabular}
}
\end{minipage}
\end{center}
\end{table}\vspace{-1ex}

\begin{figure*}[bhtp]

\begin{minipage}{0.13\linewidth}
  \centerline{\includegraphics[width=1.99cm]{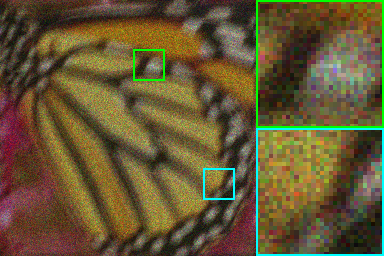}}
  \centerline{\small (a) Blur}
\end{minipage}
\hfill
\begin{minipage}{0.13\linewidth}
  \centerline{\includegraphics[width=1.99cm]{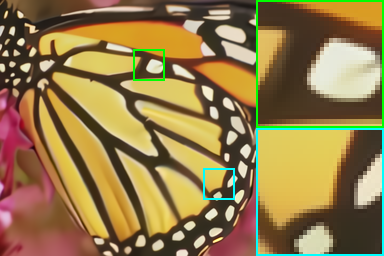}}
  \centerline{\small (b) DPIR}
\end{minipage}
\hfill
\begin{minipage}{0.13\linewidth}
  \centerline{\includegraphics[width=1.99cm]{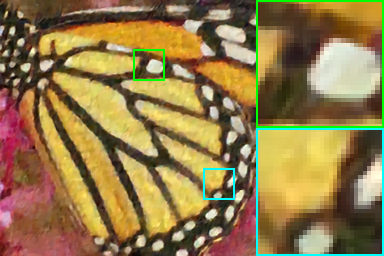}}
  \centerline{\small (c) RMMO-DRS}
\end{minipage}
\hfill
\begin{minipage}{0.13\linewidth}
  \centerline{\includegraphics[width=1.99cm]{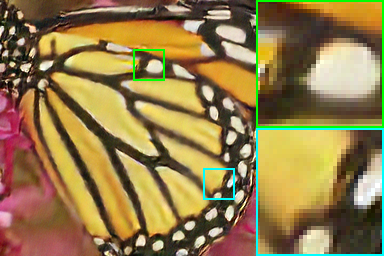}}
  \centerline{\small (d) Prox-DRS}
\end{minipage}
\hfill
\begin{minipage}{0.13\linewidth}
  \centerline{\includegraphics[width=1.99cm]{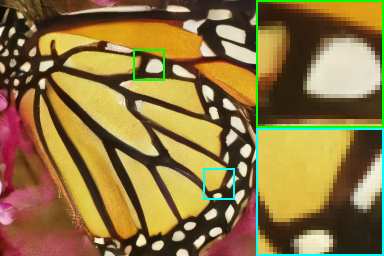}}
  \centerline{\small (e) DPS}
\end{minipage}
\hfill
\begin{minipage}{0.13\linewidth}
  \centerline{\includegraphics[width=1.99cm]{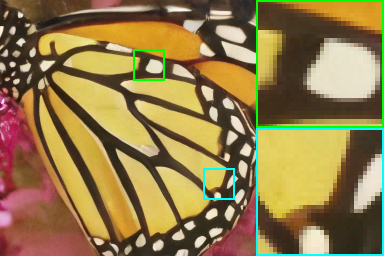}}
  \centerline{\small (f) DiffPIR}
\end{minipage}
\hfill
\begin{minipage}{0.13\linewidth}
  \centerline{\includegraphics[width=1.99cm]{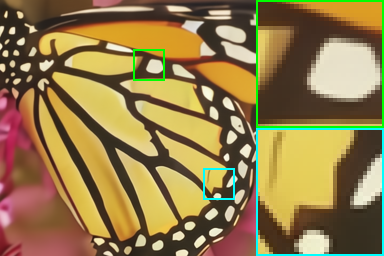}}
  \centerline{\small (g) SNORE}
\end{minipage}
\hfill
\begin{minipage}{0.13\linewidth}
  \centerline{\includegraphics[width=1.99cm]{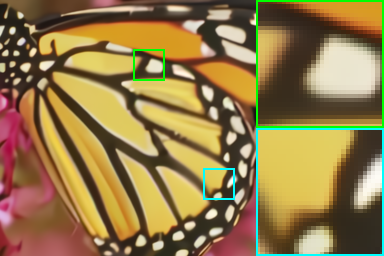}}
  \centerline{\small (h) PnPI-HQS}
\end{minipage}
\hfill
\begin{minipage}{0.13\linewidth}
  \centerline{\includegraphics[width=1.99cm]{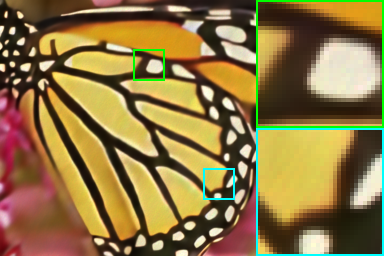}}
  \centerline{\small (i) B-RED}
\end{minipage}
\hfill
\begin{minipage}{0.13\linewidth}
  \centerline{\includegraphics[width=1.99cm]{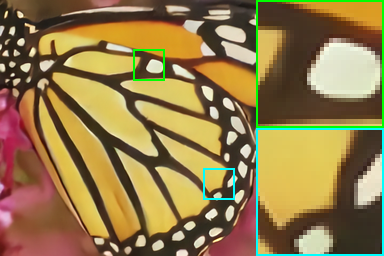}}
  \centerline{\small (j) CoCo-ADMM (k)}
\end{minipage}
\hfill
\begin{minipage}{0.13\linewidth}
  \centerline{\includegraphics[width=1.99cm]{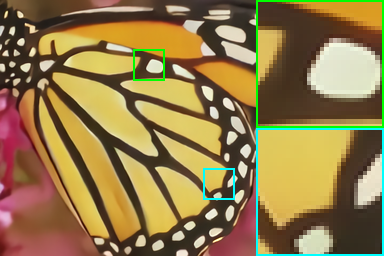}}
  \centerline{\small\ \ \ \ \   CoCo-PEGD}
\end{minipage}
\hfill
\begin{minipage}{0.13\linewidth}
  \centerline{\includegraphics[width=1.99cm]{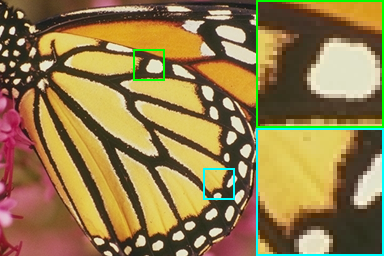}}
  \centerline{\small (l) Clean}
\end{minipage}
\hfill
\begin{minipage}{0.13\linewidth}
  \centerline{\includegraphics[width=1.99cm]{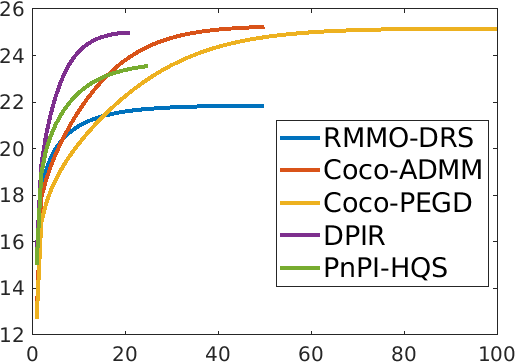}}
  \centerline{\small (m) PSNR}
\end{minipage}
\hfill
\begin{minipage}{0.13\linewidth}
  \centerline{\includegraphics[width=1.99cm]{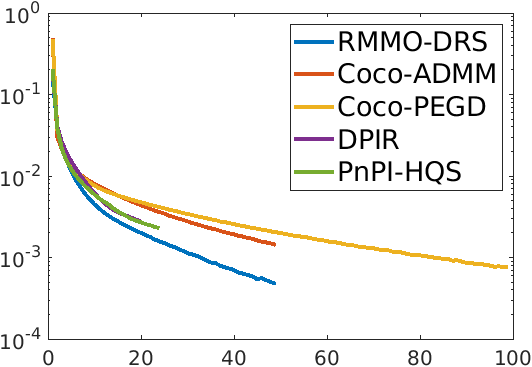}}
  \centerline{\small (n) Relative error}
\end{minipage}
\centering\caption{Deconvolution results by different methods on the image `Butterfly' from Set3c with kernel $2$ and $p=50$ Poisson noises. (a) Blur image. (b) DPIR, PSNR=$25.00$dB. (c) RMMO-DRS, PSNR=$21.86$dB. (d) Prox-DRS, PSNR=$22.57$dB. (e) DPS, PSNR=$21.47$dB. (f) DiffPIR, PSNR=$22.66$dB. (g) SNORE, PSNR=$25.14$dB. (h) PnPI-HQS, PSNR=$23.56$dB. (i) B-RED, PSNR=$ 23.34$dB. (j) CoCo-ADMM, PSNR=$25.25$dB. (k) CoCo-PEGD, PSNR=$25.17$dB. (l) Clean image. (m) PSNR curves. (n) Relative error curves. $x$-axis denotes the iteration number.
}
\label{fig butterfly}
\end{figure*}

\textbf{Photon limited deconvolution.} In this task, $K$ in (\ref{our model}) is the blur kernel. We use $8$ real-world camera shake kernels by \cite{levin2009understanding}, see Fig. \ref{kernel} in Appendix \ref{appendix deblurring}.

We compare our methods with some close related PnP methods, including DPIR \cite{zhang2021plug}, which applies PnP-HQS method with decreasing step size. Please note that DPIR is not guaranteed to converge; RMMO-DRS \cite{terris2020building}, which uses the DRS method with a firmly non-expansive denoiser RMMO; B-RED \cite{hurault2024convergent}, which uses the gradient descent method with a Bregman denoiser; PnPI-HQS \cite{AHQS}, which uses the Ishikawa HQS method with a strictly pseudo-contractive denoiser; SNORE \cite{renaud2024plug}, which proposes the novel stochastic denoising regularization by iteratively adding Gaussian noises. We also compare two state-of-the-art diffusion-based methods: DPS \cite{chung2023diffusion} and DiffPIR \cite{zhu2023denoising}. The average PSNR and SSIM values on CBSD68 are summarized in Table \ref{tab deblur}.

\begin{figure*}[htbp]

\begin{minipage}{0.16\linewidth}
  \centerline{\includegraphics[width=2.34cm]{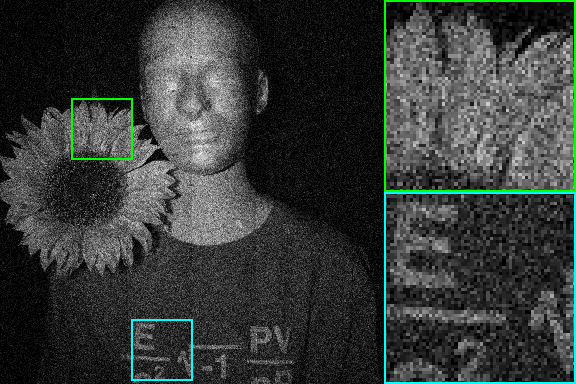}}
  \centerline{\small (a) Reflectivity}
\end{minipage}
\hfill
\begin{minipage}{0.16\linewidth}
  \centerline{\includegraphics[width=2.34cm]{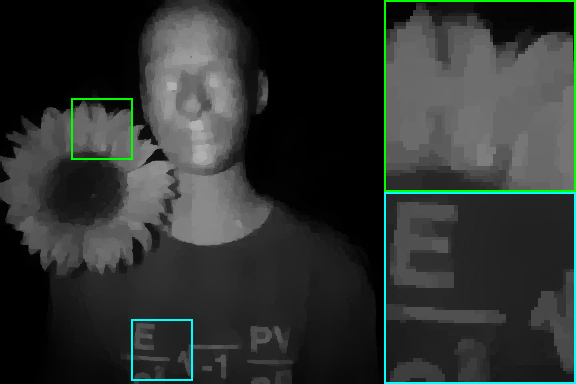}}
  \centerline{\small (b) Photon TV}
\end{minipage}
\hfill
\begin{minipage}{0.16\linewidth}
  \centerline{\includegraphics[width=2.34cm]{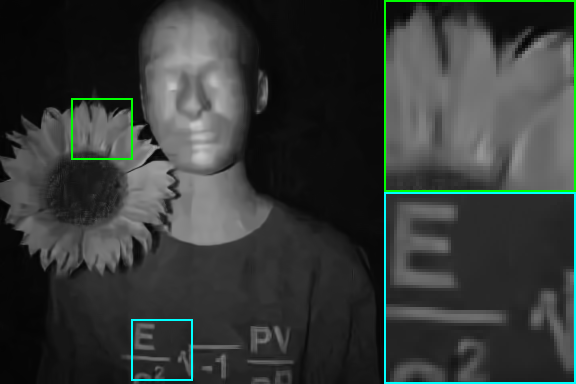}}
  \centerline{\small (c) BM3D-VST}
\end{minipage}
\hfill
\begin{minipage}{0.16\linewidth}
  \centerline{\includegraphics[width=2.34cm]{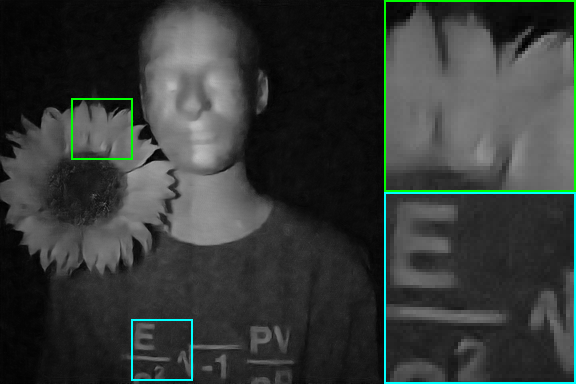}}
  \centerline{\small (d) DnCNN}
\end{minipage}
\hfill
\begin{minipage}{0.16\linewidth}
  \centerline{\includegraphics[width=2.34cm]{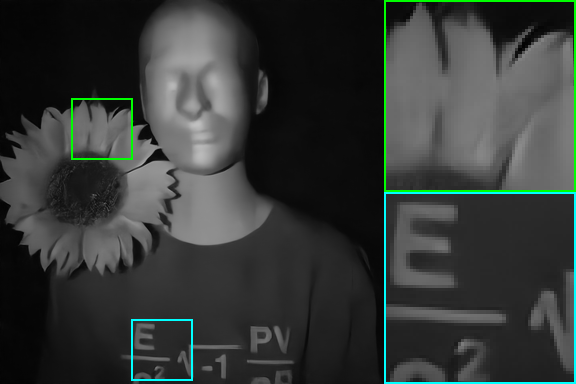}}
  \centerline{\small (e) VDIR}
\end{minipage}
\hfill
\begin{minipage}{0.16\linewidth}
  \centerline{\includegraphics[width=2.34cm]{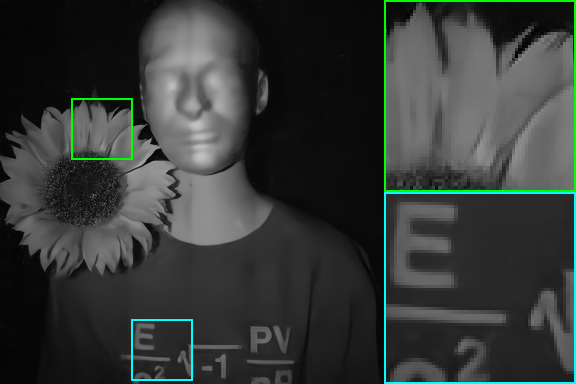}}
  \centerline{\small (f) VBDNet}
\end{minipage}
\hfill
\begin{minipage}{0.16\linewidth}
  \centerline{\includegraphics[width=2.34cm]{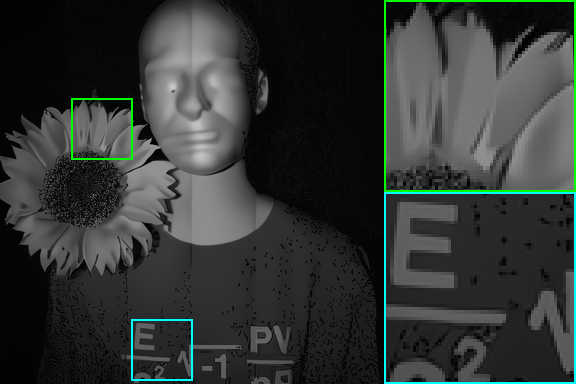}}
  \centerline{\small (g) DPIR}
\end{minipage}
\hfill
\begin{minipage}{0.16\linewidth}
  \centerline{\includegraphics[width=2.34cm]{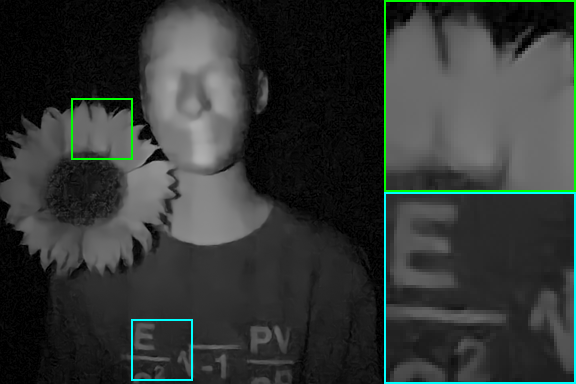}}
  \centerline{\small (h) RMMO-DRS}
\end{minipage}
\hfill
\begin{minipage}{0.16\linewidth}
  \centerline{\includegraphics[width=2.34cm]{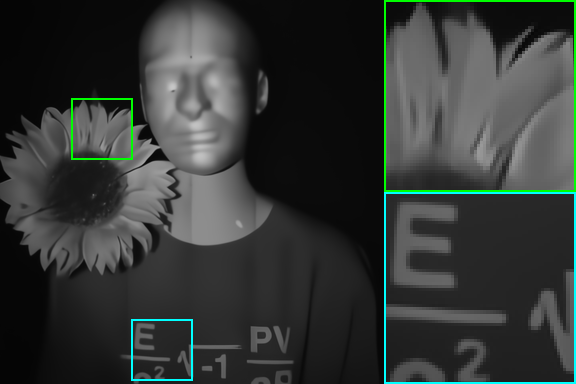}}
  \centerline{\small (i) PnPI-HQS}
\end{minipage}
\hfill
\begin{minipage}{0.16\linewidth}
  \centerline{\includegraphics[width=2.34cm]{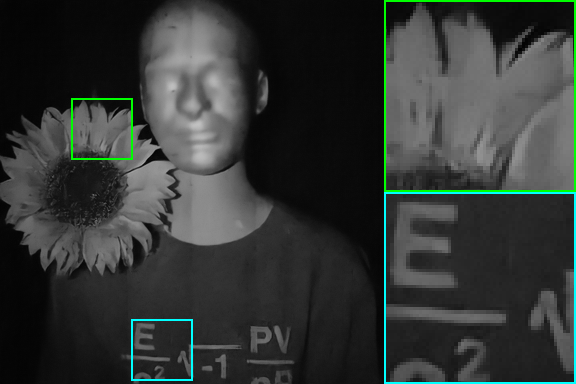}}
  \centerline{\small (j) CoCo-ADMM}
\end{minipage}
\hfill
\begin{minipage}{0.16\linewidth}
  \centerline{\includegraphics[width=2.34cm]{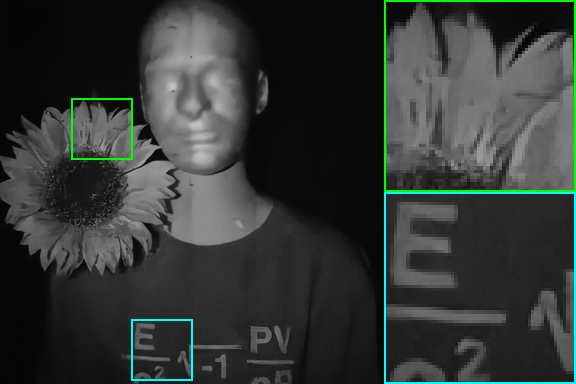}}
  \centerline{\small (k) CoCo-PEGD}
\end{minipage}
\hfill
\begin{minipage}{0.16\linewidth}
  \centerline{\includegraphics[width=2.34cm]{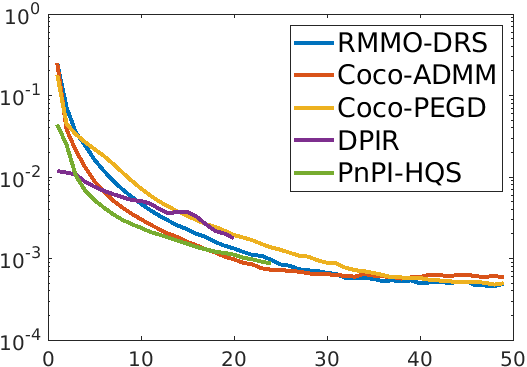}}
  \centerline{\small (l) Relative error}
\end{minipage}
\centering\caption{Single photon imaging results in a real-world low light setting by different methods. 
}
\label{fig single}
\end{figure*}

We show the visual results in Fig. \ref{fig butterfly}. In Fig. \ref{fig butterfly} (a), the image `Butterfly' is severely degraded. In this setting, many methods fail to recover the clear edges, see Figs. \ref{fig butterfly} (c)-(i). Compared with DPIR and SNORE, the enlarged parts by CoCo-ADMM and CoCo-PEGD are closer to the potential clean image, see Figs. \ref{fig butterfly} (b), (e), (h), (i). \vspace{-2ex}

\section{Conclusion and Limitation}\label{sec 6}\vspace{-1ex}
This paper introduces a novel cocoercive conservative assumption on the denoiser. Cocoerciveness is weaker than the existing assumptions, and is less restrictive for a deep denoiser. \rev{Conservativeness is encouraged by a local Jacobian-symmetry penalty motivated by Helmholtz decomposition.} We propose a novel training strategy that incorporates a Hamiltonian regularization term and a spectral regularization term, which  encourages a cocoercive conservative (CoCo) Gaussian denoiser. Theoretically, CoCo denoiser is proved to be a proximal operator of an implicit weakly convex prior function. \rev{Stationary-point convergence results are given under the stated assumptions.} \rev{The results apply to other inverse problems when their fidelity and implicit prior satisfy these assumptions.} Extensive experimental results demonstrate that the proposed CoCo-PnP methods achieve competitive performance in terms of both visual quality and quantitative measures.

The main limitation of this paper is that the proposed CoCo denoiser exhibits a slightly worse Gaussian denoising performance compared to the non-regularized denoiser, see Table \ref{denoising performance}. \rev{The regularization trades some denoising performance for encouraging the properties assumed in Theorems~\ref{thm 1}--\ref{main thm 3}.} Another limitation is that the proposed learning strategy can only encourage the desired mathematical properties, not enforce them. In experiments, we do observe that sometimes, even after a long time spectral regularization training, the trained denoiser still violates the cocoercive property on some particular images. For conservativity, we find that when $N=30$, the symmetry error compared with standard DRUNet is smaller, but not significantly smaller. How to enforce such properties without greatly compromising the denoising performance is still an open problem. Since the paper mainly focuses on the theoretical analysis of PnP methods, we will address these limitations in future works.

\section*{Acknowledgments}
Fang Li is supported by National Key  R\&D Program of China (Nos. 2021YFA1000302), Fundamental Research Funds for the Central Universities, Natural Science Foundation of Shanghai (No. 22ZR1419500), Science and Technology Commission of Shanghai Municipality (No. 22DZ2229014), and
Natural Science Foundation of Chongqing, China (No. CSTB2023NSCQ-MSX0276). 
Tieyong Zeng is supported in part by BNBU Research Grant (No. UICR0100031, UICR0700108-25 and UICR0900003)
at Beijing Normal-Hong Kong Baptist University, Zhuhai, PR China.

\bibliography{new}
\bibliographystyle{plain}

\clearpage
\appendix

\section{Technical Appendices and Supplementary Material}
\tableofcontents



\clearpage

\subsection{Proof of Lemma \ref{lemma 1}}\label{Appendix lemma1}

In order to prove Lemma \ref{lemma 1}, we will mainly use Proposition 3.1 by \cite{bauschke2010baillon}:

\begin{proposition}[Proposition 3.1 by \cite{bauschke2010baillon}]
Let $C$ be a nonempty open convex subset of $V$, let $\mathcal{B}$ be a real Banach space, and let $\operatorname{D}:C\rightarrow \mathcal{B}$ be continuously Fr\'{e}chet differentiable on $C$. Then $\operatorname{D}$ is non-expansive if and only if $\|\operatorname{J}(x)\|_*\le 1$, $\forall x\in C, \operatorname{J}(x):=\nabla \operatorname{D}(x)$.
\end{proposition}

Now we prove Lemma \ref{lemma 1}.
\begin{proof}
In Lemma \ref{lemma 1}, we let $C=\mathcal{B}=V$, thus $C$ is open convex, and $\mathcal{B}$ is a real Banach space. In order to prove Lemma \ref{lemma 1}, we only need to prove that $\operatorname{D}$ is $\gamma$-cocoercive (\rev{$\gamma>0$}), if and only if $\forall x,y\in V$
\begin{equation}\label{eq lemma 1 1}
    \|(2\gamma \operatorname{D}-\operatorname{I})\circ(x)-(2\gamma \operatorname{D}-\operatorname{I})\circ(y)\|\le \|x-y\|.
\end{equation}
Note that (\ref{eq lemma 1 1}) is equivalent to
\begin{equation}
    4\gamma^2\|\operatorname{D}(x)-\operatorname{D}(y)\|^2 + \|x-y\|^2 - 4\gamma \langle x-y,\operatorname{D}(x)-\operatorname{D}(y)\rangle\le \|x-y\|^2,
\end{equation}
and equivalent to 
\begin{equation}
    \gamma\|\operatorname{D}(x)-\operatorname{D}(y)\|^2\le \langle x-y, \operatorname{D}(x)-\operatorname{D}(y)\rangle,
\end{equation}
which means $\operatorname{D}$ is $\gamma$-cocoercive.
\end{proof}

\subsection{Characterizations on the denoiser}\label{appendix characterization}

The first characterization of proximal operators of convex, closed, and proper functions is due to Moreau's theorem \cite{moreau1965proximite}. 
\begin{theorem}[\cite{moreau1965proximite}]\label{Moreau thm}
A map $\operatorname{D}:V\rightarrow V$ is a proximal operator of a proper, closed, convex function $F:V\rightarrow \mathbb{R}\cup\{+\infty\}$, if and only if the it holds that:
\begin{itemize}
    \item there exists a convex, closed function $\psi$ such that $\operatorname{D}(x)\in \partial \psi(x), \forall x\in V$;
    \item $\operatorname{D}$ is nonexpansive, i.e.,
    \begin{equation}
        \|\operatorname{D}(x)-\operatorname{D}(y)\|\le \|x-y\|, \forall x,y\in V.
    \end{equation}
\end{itemize}
\end{theorem}
\cite{gribonval2020characterization} generalized Moreau's theorem to the proximal operators of potentially nonconvex functions. 
\begin{theorem}[\cite{gribonval2020characterization}]\label{Nikolova thm}
    Let $\operatorname{D}:V\rightarrow V$ and $L>0$. The following are equivalent:
    \begin{itemize}
        \item there is $F:V\rightarrow\mathbb{R}\cup\{+\infty\}$ such that $\operatorname{D}(x)\in \operatorname{Prox}_F(x), \forall x\in V$, and $x\mapsto F(x)+\left(1-\frac{1}{L}\right)\frac{\|x\|^2}{2}$ is closed and convex;
        \item $\operatorname{D}$ is $L$-Lipschitz, and that there exists a closed, convex function $\psi$ such that $\operatorname{D}(x)\in \partial \psi(x),\forall x\in V$.
    \end{itemize}
\end{theorem}
When $\operatorname{D}\in \mathcal{C}^1[V]$, Gribonval and Nikolova also provided the following characterization.
\begin{theorem}[\cite{gribonval2020characterization}]\label{Nikolova thm2}
Let $\operatorname{D}:V\rightarrow V$, and $\operatorname{D}\in \mathcal{C}^1[V]$. The following properties are equivalent:
\begin{itemize}
    \item $\operatorname{D}$ is a proximal operator of a function $F:V\rightarrow\mathbb{R}\cup\{+\infty\}$;
    \item there exists a convex $\mathcal{C}^2[V]$ function $\psi$ such that $\operatorname{D}(x)=\nabla\psi(x), \forall x\in V$;
    \item the differential $\operatorname{J}(x)=\nabla \operatorname{D}(x)$ is symmetric positive semi-definite for all $x\in V$.
\end{itemize}
\end{theorem}

\subsection{More general characterization on the CoCo denoiser}\label{appendix general characterization}

The following theorem is a tight use of Theorem \ref{Nikolova thm2}. 
\begin{theorem}\label{thm 0}
Let $\operatorname{D}_\sigma\in \mathcal{C}^1[V]$. $\beta=\frac{1}{\sigma^2}$. $\operatorname{D}_\sigma$ satisfies that:\\
$\bullet$ $\operatorname{D}_\sigma$ is conservative;\\
$\bullet$ $\operatorname{D}_\sigma$ is $\gamma$-cocoercive with $\gamma\in(0,\infty)$.\\
Then, there exists a function $F:V\rightarrow \bar{\mathbb{R}}$, such that $F$ is $r$-weakly convex, where $r=\beta(1-\gamma)$, and that $\operatorname{D}_\sigma(x)\in \operatorname{Prox}_{\frac{F}{\beta}}(x), \forall x\in V$.
\end{theorem}

Note that here is a little abuse of notation here: when $\gamma>1$, $r<0$. ``$r$-weakly convexity with a negative $r$'' actually means ``$(-r)$-strongly convexity''.
\begin{proof}
When $\operatorname{D}_\sigma$ is $\gamma$-cocoercive with $\gamma>0$, by Lemma \ref{lemma 1}, $\forall x\in V$, $\|2\gamma\operatorname{J}(x)-\operatorname{I}\|_*\le 1$. \rev{Since $\operatorname{D}_\sigma$ is conservative, $\operatorname{J}(x)=\operatorname{J}^\mathrm{T}(x)$; the norm bound then implies $0\preceq\operatorname{J}(x)\preceq\gamma^{-1}\operatorname{I}$.} Therefore, $\operatorname{J}(x)$ is symmetric semi-positive for any $x\in V$. By Poincar\'{e}'s lemma (see Theorem 6.6.3 in \cite{galbis2012vector}), there exists a convex $\mathcal{C}^2[V]$ function $\psi$ such that $\operatorname{D}_\sigma(x)=\nabla \psi(x)$, $\forall x\in V$. 

\rev{Take $F(u)=\beta(\psi^*(u)-\frac12\|u\|^2)$, where $\psi^*$ is the convex conjugate of $\psi$ and $\beta=\sigma^{-2}>0$. Since $\nabla\psi$ is $\gamma^{-1}$-Lipschitz, $\psi^*$ is $\gamma$-strongly convex \cite{bauschke2017correction}. Thus $F+\frac{\beta(1-\gamma)}2\|\cdot\|^2$ is convex. Fenchel conjugacy gives $\operatorname{Prox}_{F/\beta}(x)=\arg\min_u\{\psi^*(u)-\langle x,u\rangle\}=\{\nabla\psi(x)\}$.} Now we prove that $F$ is weakly convex. Recall the resolvent form of a proximal operator: $\operatorname{Prox}_{\frac{F}{\beta}}=(\operatorname{I}+\frac{1}{\beta}\partial F)^{-1}$. Given any $x,y\in V$, choose arbitrary $u,v\in V$, $u=\operatorname{D}_\sigma(x)\in(\operatorname{I}+\frac{1}{\beta}\partial F)^{-1}(x), v=\operatorname{D}_\sigma(y)\in(\operatorname{I}+\frac{1}{\beta}\partial F)^{-1}(y)$, then
\begin{equation}\label{eq temp 0}
   \beta(x-u)\in \partial F(u), \beta(y-v)\in \partial F(v).
\end{equation}
Since $\operatorname{D}_\sigma$ is $\gamma$-cocoercive, by definition \ref{def coco},
\begin{equation}\label{eq temp 1}
    \langle x-y,\operatorname{D}_\sigma(x)-\operatorname{D}_\sigma(y)\rangle \ge \gamma \|\operatorname{D}_\sigma(x)-\operatorname{D}_\sigma(y)\|^2. 
\end{equation}
By substituting $u,v$ in (\ref{eq temp 1}), we have 
\begin{equation}\label{eq temp 6}
\begin{array}{ll}
     \langle x-y,u-v\rangle \ge \gamma \|u-v\|^2   \\
     \langle (x-u) - (y-v) + (u-v),u-v\rangle \ge \gamma \|u-v\|^2 \\
    \langle \beta(x-u) - \beta(y-v) + \beta(u-v),u-v\rangle \ge \beta\gamma \|u-v\|^2 \\
    \langle \beta(x-u) - \beta(y-v),u-v\rangle \ge \beta(\gamma-1) \|u-v\|^2\\
    \langle \beta(x-u) - \beta(y-v),u-v\rangle + \beta(1-\gamma) \|u-v\|^2\ge 0.
\end{array}
\end{equation}
Recall (\ref{eq temp 0}), we know that $F$ is $r$-weakly convex with $r=\beta(1-\gamma)$.
\end{proof}

\subsection{Proof of Theorem \ref{thm 1}}\label{Appendix thm1}
Please note that, when $\gamma\ge 1$, $\operatorname{D}_\sigma$ is firmly non-expansive and conservative, and therefore is a proximal operator of some convex function. In this case, $r(t)$ could be negative: the ``negative weakly convexity'' actually means convexity. Since a convex implicit prior function $F$ is out of interest in this paper, Theorem \ref{thm 1} focuses on the weakly convex case, that is, $\gamma\in(0,1)$.

\begin{proof}
\rev{For $t=0$ the claim is immediate. For $0<t<1$, write $\operatorname{D}_\sigma=\nabla\psi$ and $h=(1-t)\|\cdot\|^2/2+t\psi$. The prior $F=\beta(h^*-\|\cdot\|^2/2)$ is finite-valued and $\mathcal{C}^2$, since $(1-t)\operatorname{I}\preceq\nabla^2h\preceq(1-t+t/\gamma)\operatorname{I}$. Hence $\partial F=\{\nabla F\}$, and $\operatorname{D}_\sigma^t=\nabla h$ is onto $V$. The following calculations therefore control $\partial F$ on all of $V$.}

Since $\operatorname{D}_\sigma$ is $\gamma$-cocoercive with $\gamma>0$, $\operatorname{D}_\sigma^t$ is naturally $\frac{\gamma}{t+(1-t)\gamma}$-cocoercive: $\forall x,y\in V$, we have
\begin{equation}
\begin{array}{rl}
     &\langle \operatorname{D}_\sigma(x)-\operatorname{D}_\sigma(y),x-y\rangle \ge \gamma \|\operatorname{D}_\sigma(x)-\operatorname{D}_\sigma(y)\|^2\\ \vspace{1ex}
    &\displaystyle\langle  \frac{1}{t}(\operatorname{D}_\sigma^t-(1-t)\operatorname{I})\circ(x)-\frac{1}{t}(\operatorname{D}_\sigma^t-(1-t)\operatorname{I})\circ(y), x-y\rangle \\
    \ge&\gamma \|\frac{1}{t}(\operatorname{D}_\sigma^t-(1-t)\operatorname{I})\circ(x)-\frac{1}{t}(\operatorname{D}_\sigma^t-(1-t)\operatorname{I})\circ(y)\|^2\\ \vspace{1ex}
    &t\langle  \operatorname{D}_\sigma^t(x)-\operatorname{D}_\sigma^t(y),x-y\rangle -t(1-t)\|x-y\|^2
     \ge\gamma \|\operatorname{D}_\sigma^t(x)-\operatorname{D}_\sigma^t(y) - (1-t)(x-y)\|^2.
\end{array}
\end{equation}
Denote $a=\operatorname{D}_\sigma^t(x)-\operatorname{D}_\sigma^t(y)$, $b=x-y$ for convenience. Then we have 
\begin{equation}\label{eq temp 2}
\begin{array}{ll}
    t\langle a,b\rangle -t(1-t)\|b\|^2 &\ge \gamma \|a\|^2 - 2\gamma (1-t)\langle a,b\rangle + \gamma (1-t)^2 \|b\|^2,\\
    (t+2\gamma(1-t))\langle a, b\rangle &\ge \gamma \|a\|^2 + (t(1-t)+\gamma(1-t)^2)\|b\|^2.
\end{array}
\end{equation}
Now note that 
\begin{equation}\label{eq temp 3}
    \langle a,b\rangle=t\langle \operatorname{D}_\sigma(x)-\operatorname{D}_\sigma(y),x-y\rangle+(1-t)\|x-y\|^2.
\end{equation}
Since $\operatorname{D}_\sigma$ is $\gamma$-cocoercive with $\gamma>0$, we know that $\operatorname{D}_\sigma$ is $\frac{1}{\gamma}$-Lipschitz. Therefore, by Cauchy-Schwarz inequality, 
\begin{equation}\label{eq temp 4}
    \langle \operatorname{D}_\sigma(x)-\operatorname{D}_\sigma(y),x-y\rangle \le \frac{1}{\gamma}\|x-y\|^2.
\end{equation}
That is, 
\begin{equation}\label{eq temp 5}
\begin{array}{ll}
         \langle a,b\rangle &\le \left(\frac{t}{\gamma}+1-t\right)\|b\|^2,  \\
    \gamma(1-t)\langle a,b\rangle &\le t(1-t)\|b\|^2 + \gamma(1-t)^2\|b\|^2.
\end{array}
\end{equation}
By substituting (\ref{eq temp 5}) into (\ref{eq temp 2}), we have
\begin{equation}
\begin{array}{ll}
     (t+\gamma(1-t))\langle a, b\rangle &\ge \gamma \|a\|^2,
      \\ \vspace{1ex}
     \langle a, b\rangle &\ge \displaystyle\frac{\gamma}{t+\gamma(1-t)} \|a\|^2,
\end{array}
\end{equation}
which means that $\operatorname{D}_\sigma^t$ is $\frac{\gamma}{t+\gamma(1-t)}$-cocoercive. Since $\operatorname{D}_\sigma$ is conservative, $\operatorname{D}_\sigma^t$ is also conservative. By Theorem \ref{thm 0}, we know that there exists a $r$-weakly convex function $F:V\rightarrow \bar{\mathbb{R}}$, where $$r=r(t)=\beta\left(1-\frac{\gamma}{t+\gamma(1-t)}\right)=\beta\frac{t-\gamma t}{t+\gamma - \gamma t},$$
such that $\operatorname{D}_\sigma^t\in\operatorname{Prox}_{\frac{F}{\beta}}$. 

Now we prove that $\partial F$ is $L$-Lipschitz. We consider two cases seperately. \\
\textbf{Case 1:} $t\ge\frac{1-2\gamma}{2-2\gamma}$ and $t\in[0,1)$.

In this case, we need to prove that $\partial F$ is $L(t)=\beta\frac{t}{1-t}$ Lipschitz. Given $\forall x,y\in V$, choose arbitrary $u,v$, such that, $u=\operatorname{D}_\sigma^t(x)\in(\operatorname{I}+\frac{1}{\beta}\partial F)^{-1}(x), v=\operatorname{D}_\sigma^t(y)\in(\operatorname{I}+\frac{1}{\beta}\partial F)^{-1}(y)$, then
\begin{equation}
   \beta(x-u)\in \partial F(u), \beta(y-v)\in \partial F(v).
\end{equation}

Note that $a=u-v,b=x-y$. In order to prove $\partial F$ is $L(t)$-Lipschitz, we need to prove
\begin{equation}\label{appendix Lt}
\begin{array}{rl}
    \|\beta(x-u)-\beta(y-v)\|^2&\le L^2(t)\|u-v\|^2=\beta^2\displaystyle\frac{t^2}{(1-t)^2}\|u-v\|^2,  \\ \vspace{1ex}
    \|\beta(b-a)\|^2&\le \beta^2\displaystyle\frac{t^2}{(1-t)^2}\|a\|^2,\\
    \vspace{1ex}
    \|b-a\|^2&\le \displaystyle\frac{t^2}{(1-t)^2}\|a\|^2.
\end{array}
\end{equation}
Since $a=u-v=\operatorname{D}_\sigma^t(x)-\operatorname{D}_\sigma^t(y)=t(\operatorname{D}_\sigma(x)-\operatorname{D}_\sigma(y))+(1-t)(x-y)$, $b=x-y$, we have \begin{equation}
    a-b=t(\operatorname{D}_\sigma(x)-\operatorname{D}_\sigma(y))-t(x-y).
\end{equation}
Now we only need to prove
\begin{equation}
\begin{array}{ll}
    &t^2\|\operatorname{D}_\sigma(x)-\operatorname{D}_\sigma(y)\|^2 - {2t^2}\langle \operatorname{D}_\sigma(x)-\operatorname{D}_\sigma(y),x-y\rangle +t^2\|x-y\|^2  \\
    \le & \displaystyle\frac{t^2}{(1-t)^2}\left(t^2\|\operatorname{D}_\sigma(x)-\operatorname{D}_\sigma(y)\|^2+2t(1-t)\langle \operatorname{D}_\sigma(x)-\operatorname{D}_\sigma(y),x-y\rangle +(1-t)^2\|x-y\|^2\right),
\end{array}
\end{equation}
which is equivalent to prove
\begin{equation}
    \displaystyle{\frac{1-2t}{2-2t}}\|\operatorname{D}_\sigma(x)-\operatorname{D}_\sigma(y)\|^2\le\langle \operatorname{D}_\sigma(x)-\operatorname{D}_\sigma(y),x-y\rangle.
\end{equation}
{
Since $t\ge\frac{1-2\gamma}{2-2\gamma}$, we have that 
\begin{equation}
    \displaystyle\frac{1-2t}{2-2t}\le \displaystyle\frac{1-2\frac{1-2\gamma}{2-2\gamma}}{2-2\frac{1-2\gamma}{2-2\gamma}}=\displaystyle\frac{2-2\gamma-2(1-2\gamma)}{2(2-2\gamma)-2(1-2\gamma)}=\displaystyle\frac{2\gamma}{2}=\gamma.
\end{equation}
}
We already have
\begin{equation}
    \gamma\|\operatorname{D}_\sigma(x)-\operatorname{D}_\sigma(y)\|^2\le\langle \operatorname{D}_\sigma(x)-\operatorname{D}_\sigma(y),x-y\rangle.
\end{equation}
Thus, 
\begin{equation}
    \frac{1-2t}{2-2t}\|\operatorname{D}_\sigma(x)-\operatorname{D}_\sigma(y)\|^2\le    
    \gamma\|\operatorname{D}_\sigma(x)-\operatorname{D}_\sigma(y)\|^2\le\langle \operatorname{D}_\sigma(x)-\operatorname{D}_\sigma(y),x-y\rangle.
\end{equation}\\
{\textbf{Case 2: $t\le\frac{1-2\gamma}{2-2\gamma}$ and $t\in[0,1)$.}

In this case, we need to prove that $\partial F$ is $L(t)=r(t) = \beta\frac{t-t\gamma}{t+\gamma-t\gamma}$ Lipschitz. Similarly in \eqref{appendix Lt}, we need to prove
\begin{equation}
\begin{array}{rl}
    \|\beta(x-u)-\beta(y-v)\|^2&\le L^2(t)\|u-v\|^2=\displaystyle\beta^2\frac{t^2(1-\gamma)^2}{(t+\gamma-t\gamma)^2}\|u-v\|^2,  \\
    \vspace{1ex}
    \|b-a\|^2&\le \displaystyle\frac{t^2(1-\gamma)^2}{(t+\gamma-t\gamma)^2}\|a\|^2.
\end{array}
\end{equation}
Since $a=u-v=\operatorname{D}_\sigma^t(x)-\operatorname{D}_\sigma^t(y)=t(\operatorname{D}_\sigma(x)-\operatorname{D}_\sigma(y))+(1-t)(x-y)$, $b=x-y$, we have \begin{equation}
    a-b=t(\operatorname{D}_\sigma(x)-\operatorname{D}_\sigma(y))-t(x-y).
\end{equation}
Now we only need to prove
\begin{equation}
\begin{array}{ll}
    &t^2\|\operatorname{D}_\sigma(x)-\operatorname{D}_\sigma(y)\|^2 - 2t^2\langle \operatorname{D}_\sigma(x)-\operatorname{D}_\sigma(y),x-y\rangle +t^2\|x-y\|^2  \\
    \le & \displaystyle\frac{t^2(1-\gamma)^2}{(t+\gamma-t\gamma)^2}\left(t^2\|\operatorname{D}_\sigma(x)-\operatorname{D}_\sigma(y)\|^2+2t(1-t)\langle \operatorname{D}_\sigma(x)-\operatorname{D}_\sigma(y),x-y\rangle +(1-t)^2\|x-y\|^2\right).
\end{array}
\end{equation}
When $t=0$, it holds naturally. When $t\in (0,1)$, and $t\le \frac{1-2\gamma}{2-2\gamma}$, it is equivalent to prove
\begin{equation}
\begin{array}{ll}
    \|\operatorname{D}_\sigma(x)-\operatorname{D}_\sigma(y)\|^2 - 2\langle \operatorname{D}_\sigma(x)-\operatorname{D}_\sigma(y),x-y\rangle +\|x-y\|^2 \le \\
      \displaystyle\frac{(1-\gamma)^2}{(t+\gamma-t\gamma)^2}\left(t^2\|\operatorname{D}_\sigma(x)-\operatorname{D}_\sigma(y)\|^2+2t(1-t)\langle \operatorname{D}_\sigma(x)-\operatorname{D}_\sigma(y),x-y\rangle +(1-t)^2\|x-y\|^2\right),
\end{array}
\end{equation}
which is equivalent to prove
\begin{equation}
\begin{array}{ll}
    &(t+\gamma-t\gamma)^2\left[\|\operatorname{D}_\sigma(x)-\operatorname{D}_\sigma(y)\|^2 - 2\langle \operatorname{D}_\sigma(x)-\operatorname{D}_\sigma(y),x-y\rangle +\|x-y\|^2 \right] \\
    \le & (1-\gamma)^2\left[t^2\|\operatorname{D}_\sigma(x)-\operatorname{D}_\sigma(y)\|^2+2t(1-t)\langle \operatorname{D}_\sigma(x)-\operatorname{D}_\sigma(y),x-y\rangle +(1-t)^2\|x-y\|^2\right],
\end{array}
\end{equation}
that is to prove
\begin{equation}\label{temp eq0}
\begin{array}{ll}
    &\left[(t+\gamma-t\gamma)^2-t^2(1-\gamma)^2\right]\|\operatorname{D}_\sigma(x)-\operatorname{D}_\sigma(y)\|^2\\
    &-\left[2(t+\gamma-t\gamma)^2+2t(1-t)(1-\gamma)^2\right]\langle \operatorname{D}_\sigma(x)-\operatorname{D}_\sigma(y),x-y\rangle\vspace{1ex}\\ 
    \le&  \left[(1-t)^2(1-\gamma)^2-(t+\gamma-t\gamma)^2\right]\|x-y\|^2.
\end{array}
\end{equation}
Now we check each coefficient, and estimate each term carefully. For the coefficient of $\|\operatorname{D}_\sigma(x)-\operatorname{D}_\sigma(y)\|^2$, we have
\begin{equation}\label{temp eq-1}
    (t+\gamma-t\gamma)^2-t^2(1-\gamma)^2=2\gamma(1-\gamma)t+\gamma^2>0
\end{equation}
For the coefficient of $\langle \operatorname{D}_\sigma(x)-\operatorname{D}_\sigma(y),x-y\rangle$, it is obviously non-positive. Besides, we have that
\begin{equation}
\begin{array}{ll}
    &-\left[2(t+\gamma-t\gamma)^2+2t(1-t)(1-\gamma)^2\right]\\
    =&-2\left[(1-\gamma)^2t^2+2\gamma(1-\gamma)t+\gamma^2+(t-t^2)(1-\gamma)^2 \right]\\
    =& -(2-2\gamma^2)t-2\gamma^2.
\end{array}
\end{equation}
Since $\operatorname{D}_\sigma$ is $\gamma$-cocoercive, we have that 
\begin{equation}\label{temp eq2}
    \left[-(2-2\gamma^2)t-2\gamma^2\right]\langle \operatorname{D}_\sigma(x)-\operatorname{D}_\sigma(y),x-y\rangle\le  \left[-\gamma(2-2\gamma^2)t-2\gamma^3\right]\|\operatorname{D}_\sigma(x)-\operatorname{D}_\sigma(y)\|^2.
\end{equation}
For the coefficient of $\|x-y\|^2$, note that when $\gamma\in(0,1)$, $t\in[0,1)$, and $t\le \frac{1-2\gamma}{2-2\gamma}$, we have  
\begin{equation}\label{temp eq3}
    (1-t)^2(1-\gamma)^2-(t+\gamma-t\gamma)^2=(2\gamma-2)t-2\gamma+1\ge(2\gamma-2)\frac{1-2\gamma}{2-2\gamma}-2\gamma+1\ge 0.
\end{equation}

Combining \eqref{temp eq-1}-\eqref{temp eq3}, to prove \eqref{temp eq0}, we only need to prove that
\begin{equation}
\begin{array}{rl}
    \left[2\gamma(1-\gamma)t+\gamma^2 - \gamma(2-2\gamma^2)t-2\gamma^3\right]\|\operatorname{D}_\sigma(x)-\operatorname{D}_\sigma(y)\|^2&\le \left[(2\gamma-2)t-2\gamma+1\right]\|x-y\|^2\vspace{1ex}\\
     \iff\left[(2\gamma^3-2\gamma^2)t+\gamma^2-2\gamma^3\right]\|\operatorname{D}_\sigma(x)-\operatorname{D}_\sigma(y)\|^2&\le\left[(2\gamma-2)t-2\gamma+1\right]\|x-y\|^2.
\end{array}
\end{equation}
Since $\operatorname{D}_\sigma$ is $\frac{1}{\gamma}$-Lipschitz, we only need to prove that 
\begin{equation}
\begin{array}{lrl}
 &\displaystyle\frac{1}{\gamma^2}\left[(2\gamma^3-2\gamma^2)t+\gamma^2-2\gamma^3\right] &\le(2\gamma-2)t-2\gamma+1 \vspace{1ex}\\
 \iff & (2\gamma-2)t+1-2\gamma&\le (2\gamma-2)t-2\gamma+1.
\end{array}
\end{equation}
This completes the proof.}
\end{proof}
Please note that any Lipschitz operator must be single-valued. Thus, there is actually only one element in $\partial F(x)$, given any $x\in V$.

\subsection{Calculation for $\|(2\gamma\operatorname{J}-\operatorname{I})\|_*$ and $\|\operatorname{J}-\operatorname{J}^\mathrm{T}\|_*$}\label{appendix training}

We mainly use the power iterative method \cite{golub2013matrix} to calculate the spectral norm of a given operator $\operatorname{A}$.

\begin{algorithm}
\caption{Power iterative method}
\label{alg power}
\begin{algorithmic}
\STATE{Given $q^0$ with $\|q^0\|=1, \operatorname{A},N$}
\FOR{$n=1:N$}
\STATE{\rev{$z^n=\operatorname{A}^\top\operatorname{A}q^{n-1}$}}
\STATE{$q^n=\frac{z^n}{\|z^n\|}$}
\ENDFOR
\STATE{Return \rev{$\lambda^N=\|\operatorname{A}q^N\|$}} 
\end{algorithmic}
\end{algorithm}

By Algorithm \ref{alg power}, we can \rev{estimate} the spectral norm of $\operatorname{A}$. Note that in Algorithm \ref{alg power}, we need to calculate the matrix-vector products \rev{$\operatorname{A}q$ and $\operatorname{A}^\top q$}, given any $q\in V$.

Given a denoiser $\operatorname{D}_\sigma$, $\forall x\in V$, $\operatorname{J}(x)=\nabla \operatorname{D}_\sigma(x)$ is the Jacobian matrix at the point $x$. To specify, let $x=[x_1,x_2,\dots,x_n]^\mathrm{T}\in V=\mathbb{R}^n,$ and $y=[y_1,y_2,\dots,y_n]^\mathrm{T}=\operatorname{D}_\sigma(x)$ be the output after $\operatorname{D}_\sigma$. Then the Jacobian matrix $\operatorname{J}(x^*)$ at a given point $x^*$ is:

\begin{equation}
\operatorname{J}(x^*)=\left.
\begin{bmatrix}
  \displaystyle\frac{\partial y_1}{\partial x_1} & \displaystyle\frac{\partial y_1}{\partial x_2} & \cdots & \displaystyle\frac{\partial y_1}{\partial x_n} \\
  \displaystyle\frac{\partial y_2}{\partial x_1} & \displaystyle\frac{\partial y_2}{\partial x_2} & \cdots & \displaystyle\frac{\partial y_2}{\partial x_n} \\
  \vdots & \vdots & \ddots & \vdots \\
  \displaystyle\frac{\partial y_n}{\partial x_1} & \displaystyle\frac{\partial y_n}{\partial x_2} & \cdots & \displaystyle\frac{\partial y_n}{\partial x_n}
\end{bmatrix}\right|_{x=x^*}.
\end{equation}

We start with $a=[a_1,a_2,\dots,a_n]^\mathrm{T}=\operatorname{J}^\mathrm{T}(x)v$, where $v=[v_1,v_2,\dots,v_n]^\mathrm{T}$ is any vector. By basic linear algebra, we know that 
\begin{equation}
    a_i=\sum\limits_{j=1}^n\displaystyle\frac{\partial y_j}{\partial x_i} v_j, \ \forall i=1,2,\dots,n,
\end{equation}
which implies 
\begin{equation}\label{eq temp 9}
    a=\operatorname{J}^\mathrm{T}(x)v=\displaystyle\frac{\partial \langle y,v\rangle}{\partial x}.
\end{equation}

The calculation of $b=\operatorname{J}(x)v$ is a bit more complex. We will use the so-called double back-propagation technique: since $b=\operatorname{J}(x)v$, we have
\begin{equation}
    b_k = \sum\limits_{i=1}^n\displaystyle\frac{\partial y_k}{\partial x_i}v_i, \ \forall k=1,2,\dots,n.
\end{equation}

Now consider $w=[w_1,w_2,\dots,w_n]^\mathrm{T}$. Let $t=\operatorname{J}^\mathrm{T}(x)w$. Thus $t_i=\sum\limits_{j=1}^n \displaystyle\frac{\partial y_j}{\partial x_i}w_j$. 

Let $z=\displaystyle\frac{\partial \langle t,v\rangle}{\partial w}$, then 
\begin{equation}
    z=\displaystyle\frac{\partial \langle t,v\rangle}{\partial w}=\displaystyle\frac{\partial}{\partial w}\left( \sum\limits_{i=1}^n v_i\left( \sum\limits_{j=1}^n \frac{\partial y_j}{\partial x_i}w_j \right) \right).
\end{equation}
The $k$th element of $z$ is:
\begin{equation}
    z_k=\displaystyle\frac{\partial}{\partial w_k}\left( \sum\limits_{i=1}^n v_i\left( \sum\limits_{j=1}^n \frac{\partial y_j}{\partial x_i}w_j \right) \right)=\sum\limits_{i=1}^n v_i \displaystyle\frac{\partial y_k}{\partial x_i}, \ \forall k=1,2,\dots, n.
\end{equation}
That is, 
\begin{equation}\label{eq temp 10}
    z=\operatorname{J}(x)v=\displaystyle\frac{\partial \langle t,v\rangle}{\partial w}=\displaystyle\frac{\partial \langle \frac{\partial\langle y,w\rangle}{\partial x},v\rangle}{\partial w}, \ \forall w\in \mathbb{R}^n.
\end{equation}

The AutoGrad toolbox in Pytorch \cite{paszke2017automatic} allows the calculation for $\frac{\partial \langle \cdot,\cdot\rangle}{\partial\cdot}$. The pseudo-codes in Pytorch can be:

\begin{lstlisting}[language=Python,basicstyle=\ttfamily\color{revisionblue}]
# Calculate J(x)v
def _jacobian_vec(y, x, v):
    w = torch.ones_like(x, requires_grad=True)
    t = torch.autograd.grad(y, x, w, create_graph=True)[0]
    return torch.autograd.grad(t, w, v, create_graph=True)[0]
\end{lstlisting}

\begin{lstlisting}[language=Python]
# Calculate J^T(x)v
def _jacobian_transpose_vec(y, x, v):
    return torch.autograd.grad(y, x, v, create_graph=True)[0]
\end{lstlisting}

By (\ref{eq temp 9}) and (\ref{eq temp 10}), one can calculate $\|(2\gamma\operatorname{J}-\operatorname{I})\|_*$ and $\|\operatorname{J}-\operatorname{J}^\mathrm{T}\|_*$ by Algorithm \ref{alg power}.

\subsection{Proof of Theorem \ref{main thm}}\label{appendix main thm}
We will make use of the Lyapunov function $L_\beta$ for (\ref{PnP-ADMM}) according to \cite{ li2016douglas, themelis2020douglas, hurault2022proximal}:

\begin{equation}\label{envelope}
    L_\beta(u,v,b)=F(v)+G(u;f)+\beta\langle b, u-v\rangle +\frac{\beta}{2}\|u-v\|^2.
\end{equation}

We will first prove in part 1 that an important value for $L_\beta(u,v,b)$ is positive whenever \rev{$0<t<\min\{t_0,\gamma/(1-\gamma)\}$}, where $t_0$ is the positive root of the characteristic equation in (\ref{eq t gamma}). Then, we will prove in part 2 that $L_\beta$ is non-increasing with the iteration number $k$. Finally, we will prove in part 3 that CoCo-ADMM iteration in (\ref{PnP-ADMM}) converges globally to a stationary point of (\ref{our model}).

\begin{proof}
Let $h(t)=(2-2\gamma)t^3+\gamma t^2+2\gamma t - \gamma$, where $\gamma\in(0,1)$. Note that $h$ is obviously smooth, and $h(0)=-\gamma<0$, $h(\infty)=\infty$. Also note that when $t>0$,
\begin{equation}
    h'(t)=\rev{3}(2-2\gamma)t^2+2\gamma t+2\gamma>0.
\end{equation}
Therefore, there exists a unique $t_0>0$, such that $h(t_0)=0$, $h(t)>0$ if $t>t_0$, and $h(t)<0$ if $t\in[0,t_0)$.\\
{\textbf{Part 1:}

We consider a characteristic value for $L_\beta(u,v,b)$ in (\ref{envelope}): $\frac{\beta}{2}-\frac{r}{2}-\frac{L^2}{\beta}$.

By Theorem \ref{thm 1}, if $t\in[0,1)$ and $t\ge \frac{1-2\gamma}{2-2\gamma}$, we have that $r=\beta\frac{t-\gamma t}{t+\gamma - t\gamma}$ and $L=\beta \frac{t}{1-t}$. Thus we have 
\begin{equation}
\begin{array}{ll}
   &\displaystyle\frac{\beta}{2}-\frac{r}{2}-\frac{L^2}{\beta}
   = \displaystyle\frac{\beta}{2}-\frac{\beta (t-\gamma t)}{2(t+\gamma-t\gamma)} - \frac{\beta t^2}{(1-t)^2} \\ \vspace{1ex}
   =& \displaystyle \frac{\beta}{2}\left(1-\frac{t-\gamma t}{t+\gamma - t\gamma} - \frac{2t^2}{(1-t)^2}\right)
   \\ \vspace{1ex}
   =& \displaystyle \frac{\beta}{2(t+\gamma-t\gamma)(1-t)^2}\left(\gamma(1-t)^2 - 2t^2(t+\gamma-t\gamma)\right)
   \\ \vspace{1ex}
   =& -\displaystyle \frac{\beta}{2(t+\gamma-t\gamma)(1-t)^2}\left((2-2\gamma)t^3 + \gamma t^2 +2\gamma t-\gamma\right).
\end{array}
\end{equation}
When $0<t<t_0$, where $t_0$ is the positive root of the characteristic equation in (\ref{eq t gamma}), $\frac{\beta}{2}-\frac{r}{2}-\frac{L^2}{\beta}> 0$ holds.

\begin{revision}
If $t\le(1-2\gamma)/(2-2\gamma)$, then $L=r\ge\beta t/(1-t)$. Put $\rho=r/\beta=t(1-\gamma)/(\gamma+t(1-\gamma))$. In this case,
\[
\frac\beta2-\frac r2-\frac{L^2}\beta
=\frac\beta2(1+\rho)(1-2\rho)>0
\quad\Longleftrightarrow\quad t<\frac{\gamma}{1-\gamma}.
\]
Thus imposing both strict bounds suffices in the two cases. Write $c=\beta/2-r/2-L^2/\beta>0$.
\end{revision}
}\\

\textbf{Part 2:}

Now we prove that $L_\beta(u^k,v^k,b^k)$ is non-increasing \rev{for $k\ge1$}. Before that, we show two important formulas. For $v^{k+1}$, by the first-order optimal condition, we know that 
\begin{equation}\label{formula 1}
    \beta b^{k+1}=-\beta(v^{k+1}-u^{k+1}-b^k)\in \partial F(v^{k+1}).
\end{equation}
\rev{In particular, $\beta\|b^{k+1}-b^k\|\le L\|v^{k+1}-v^k\|$ for $k\ge1$.} Similarly, for $u^{k+1}$, we have
\begin{equation}\label{formula 2}
    -\beta(u^{k+1}-v^k+b^k)\in \partial G(u^{k+1};f).
\end{equation}

In order to prove that $L_\beta(u^k,v^k,b^k)$ is non-increasing, we decompose $L_\beta(u^k,v^k,b^k)-L_\beta(u^{k+1},v^{k+1},b^{k+1})$ into two parts:

\begin{equation}
\begin{array}{ll}
    &L_\beta(u^k,v^k,b^k)-L_\beta(u^{k+1},v^{k+1},b^{k+1})\\
    =&L_\beta(u^k,v^k,b^k)-L_\beta(u^{k+1},v^k,b^k) + L_\beta(u^{k+1},v^k,b^k)-L_\beta(u^{k+1},v^{k+1},b^{k+1}),
    \end{array}
\end{equation}
and estimate them separately as follows. By the convexity of $G$ and the iteration form in (\ref{PnP-ADMM}), we have  
\begin{equation}\label{u}
\begin{array}{ll}
    &L_\beta(u^k,v^k,b^k)-L_\beta(u^{k+1},v^k,b^k)  \\
     =&G(u^k;f)-G(u^{k+1};f)+\beta\langle b^k,u^k-u^{k+1}\rangle+\displaystyle\frac{\beta}{2}\|u^k-v^k\|^2-\frac{\beta}{2}\|u^{k+1}-v^k\|^2\\
     =&G(u^k;f)-G(u^{k+1};f)-\langle -\beta(u^{k+1}-v^k+b^k),u^k-u^{k+1}\rangle +\beta \langle v^k-u^{k+1},u^k-u^{k+1}\rangle\\ &+\displaystyle\frac{\beta}{2}\|u^k-v^k\|^2-\frac{\beta}{2}\|u^{k+1}-v^k\|^2\\
     \ge & 0+\beta\langle v^k-u^{k+1},u^k-u^{k+1}\rangle + \displaystyle\frac{\beta}{2}\langle u^k-u^{k+1}, u^k+u^{k+1}-2v^k\rangle\\
     =&\beta\langle u^k-u^{k+1},v^k-u^{k+1}-v^k+\displaystyle\frac{u^k+u^{k+1}}{2}\rangle\\
     =&\displaystyle\frac{\beta}{2}\|u^k-u^{k+1}\|^2.
\end{array}
\end{equation}

By the $r$-weakly convexity of $F$, $\forall x,y\in V$, $f_y\in\partial F(y)$, we have:
\begin{equation}\label{eq temp 7}
    F(x)-F(y)\ge \langle f_y,x-y\rangle -\frac{r}{2}\|x-y\|^2.
\end{equation}
By the $L$-Lipschitz property of $\partial F$, $\forall x,y\in V$, $f_y\in\partial F(y)$, we have:
\begin{equation}\label{eq temp 8}
    F(x)-F(y)\le \langle f_y,x-y\rangle + \frac{L}{2}\|x-y\|^2.
\end{equation}

Combining (\ref{eq temp 7}) and (\ref{eq temp 8}), we can obtain that
\begin{equation}\label{vb}
    \begin{array}{ll}
        &L_\beta(u^{k+1},v^k,b^k)-L_\beta(u^{k+1},v^{k+1},b^{k+1})  \vspace{0.5ex} \\
        =& F(v^k)-F(v^{k+1})+\beta\langle b^k,u^{k+1}-v^k\rangle-\beta\langle b^{k+1},u^{k+1}-v^{k+1}\rangle \vspace{0.5ex} \\
        &+ \displaystyle\frac{\beta}{2}\|u^{k+1}-v^k\|^2-\frac{\beta}{2}\|u^{k+1}-v^{k+1}\|^2\vspace{0.5ex}\\
        =&F(v^k)-F(v^{k+1})+\beta\langle b^k,u^{k+1}-v^k\rangle-\beta\langle b^k,u^{k+1}-v^{k+1}\rangle -\beta\langle u^{k+1}-v^{k+1},u^{k+1}-v^{k+1}\rangle\vspace{0.5ex}\\
        &+\displaystyle\frac{\beta}{2}\|v^k-v^{k+1}\|^2+\beta\langle u^{k+1}-v^{k+1},v^{k+1}-v^k\rangle\vspace{0.5ex}\\
        =&F(v^k)-F(v^{k+1})+\beta\langle b^k,v^{k+1}-v^k\rangle - \beta\|u^{k+1}-v^{k+1}\|^2\vspace{0.5ex}\\
        &+\displaystyle\frac{\beta}{2}\|v^k-v^{k+1}\|^2+\beta\langle u^{k+1}-v^{k+1},v^{k+1}-v^k\rangle\vspace{0.5ex}\\
        =&\rev{F(v^k)-F(v^{k+1})+\langle\beta b^{k+1},v^{k+1}-v^k\rangle-\beta\|u^{k+1}-v^{k+1}\|^2+\frac\beta2\|v^k-v^{k+1}\|^2}\\
        \ge&\displaystyle-\frac{r}{2}\|v^k-v^{k+1}\|^2-\beta\|u^{k+1}-v^{k+1}\|^2+\frac{\beta}{2}\|v^k-v^{k+1}\|^2\vspace{0.5ex}\\
        =&\displaystyle\left(\frac{\beta}{2}-\frac{r}{2}\right)\|v^k-v^{k+1}\|^2-\beta\|b^k-b^{k+1}\|^2\vspace{0.5ex}\\
        \ge & \displaystyle\left(\frac{\beta}{2}-\frac{r}{2}\right)\|v^k-v^{k+1}\|^2-\frac{L^2}{\beta}\|v^k-v^{k+1}\|^2\vspace{0.5ex}\\
        =&\displaystyle\left(\frac{\beta}{2}-\frac{r}{2}-\frac{L^2}{\beta}\right)\|v^k-v^{k+1}\|^2.
    \end{array}
\end{equation}

Note that the second `$=$' comes from the cosine rule, the first `$\ge$' follows from the $r$-weakly convexity of $F$ as in (\ref{eq temp 7}), and the second `$\ge$' results from the $L$-Lipschitz of $\partial F$ as in (\ref{eq temp 8}).

Combining (\ref{u}) and (\ref{vb}), we get
\begin{equation}
    L_\beta(u^k,v^k,b^k)-L_\beta(u^{k+1},v^{k+1},b^{k+1})\ge \displaystyle\frac{\beta}{2}\|u^k-u^{k+1}\|^2+\left(\frac{\beta}{2}-\frac{r}{2}-\frac{L^2}{\beta}\right)\|v^k-v^{k+1}\|^2\ge 0,
\end{equation}
that is, $L_\beta(u^k,v^k,b^k)$ is non-increasing.

\begin{revision}
Now we prove that $\{(u^k,v^k,b^k)\}$ is bounded. Since $F$ is coercive and $G$ is bounded below, $F+G$ is coercive and bounded below.
\end{revision}

Along with $\beta b^k\in \partial F(v^k)$ and the property that $\partial F$ is $L$-Lipschitz, we arrive at
\begin{equation}
\begin{array}{ll}
     L_\beta(u^k,v^k,b^k)&=\displaystyle F(v^k)+G(u^k;f)+\beta\langle b^k,u^k-v^k\rangle +\frac{\beta}{2}\|u^k-v^k\|^2  \vspace{0.5ex}\\
     &\ge \displaystyle F(u^k)+G(u^k;f)-\frac{L}{2}\|u^k-v^k\|^2+\frac{\beta}{2}\|u^k-v^k\|^2\vspace{0.5ex}.\\
\end{array}
\end{equation}
\rev{Part 1 gives $\beta/2-r/2-L^2/\beta>0$, hence $L<\beta$ in both cases.} Therefore, 
\begin{equation}
    \begin{array}{ll}
         &\displaystyle F(u^k)+G(u^k;f)-\frac{L}{2}\|u^k-v^k\|^2+\frac{\beta}{2}\|u^k-v^k\|^2\vspace{0.5ex}\\
        = &\displaystyle F(u^k)+G(u^k;f) + \frac{1}{2}(\beta-L)\|u^k-v^k\|^2\rev{\ge\inf(F+G)>-\infty}.
    \end{array}
\end{equation}
\rev{Monotonicity and this lower bound first bound $u^k$, then $u^k-v^k$ and $v^k$. Finally, $\beta b^k=\nabla F(v^k)$ bounds $b^k$.} \\
\textbf{Part 3:}

Define $q^{k+1}$ as follows:
\begin{equation}
    q^{k+1} = [\beta(b^{k+1}-b^k+v^k-v^{k+1}), \beta(v^{k+1}-u^{k+1}), \beta(b^{k+1}-b^{k})].
\end{equation}
Define $\partial L_\beta(u,v,b)=[\partial_u L_\beta, \partial_v L_\beta, \partial_b L_\beta]$. By the formulas (\ref{formula 1})-(\ref{formula 2}), we know that 
\begin{equation}
    q^{k+1}\in \partial L_\beta(u^{k+1},v^{k+1},b^{k+1}).
\end{equation}
Note that 
\begin{equation}
    \|\beta(b^{k+1}-b^k+v^k-v^{k+1})\|\le \beta\|b^k-b^{k+1}\|+\beta\|v^k-v^{k+1}\|\le \displaystyle L \|v^k-v^{k+1}\|+\beta\|v^k-v^{k+1}\|,
\end{equation}
and that
\begin{equation}
    \|\beta(v^{k+1}-u^{k+1})\|=\beta\|b^k-b^{k+1}\|\le L\|v^k-v^{k+1}\|,
\end{equation}
we arrive at 
\begin{equation}
    \|q^{k+1}\|\le C\|v^k-v^{k+1}\|,
\end{equation}
where 
\begin{equation}
    C=\displaystyle 3L + \beta.
\end{equation}

\begin{revision}
By Part 2, the sequence $z^k=(u^k,v^k,b^k)$ is bounded, $L_\beta(z^k)$ converges, and its increments tend to zero. The dual increment bound gives sufficient decrease in $\|z^{k+1}-z^k\|^2$, while Part 3 gives the relative-error bound. For a subsequence $z^{n_k}\to z^*$, comparison of the minimizing $u^{n_k}$ in its proximal subproblem with $u^*$ gives $\limsup_kG(u^{n_k})\le G(u^*)$; lower semicontinuity gives the reverse inequality. Thus $L_\beta(z^{n_k})\to L_\beta(z^*)$. The assumed KL property of the joint function $L_\beta$ now permits Theorem~2.9 of \cite{attouch2013convergence} to apply: the entire sequence has finite length and converges to a stationary point $(u^*,v^*,b^*)$ of $L_\beta$.
\end{revision}

Since $(u^*,v^*,b^*)$ is a stationary point of $L_\beta$, and as a result, $q^*=0$, that is
\begin{equation}
    q^* = [0, \beta(v^*-u^*), 0]=0.
\end{equation}
Therefore,$u^*=v^*$. By CoCo-ADMM iteration in (\ref{PnP-ADMM}), we know that 
\begin{equation}
\begin{array}{ll}
         u^*&=\operatorname{Prox}_{\frac{G}{\beta}}(u^*-b^*),  \vspace{0.5ex}\\
     u^*&=\operatorname{D}_\sigma^t (u^*+b^*)\in \operatorname{Prox}_{\rev{\frac{F}{\beta}}}(u^*+b^*).
\end{array}
\end{equation}
Equivalently,
\begin{equation}
    -\beta b^*\in \partial G(u^*;f), \beta b^*\in\partial F(u^*),
\end{equation}

\begin{equation}
    0\in \partial F(u^*)+\partial G(u^*;f).
\end{equation}
Therefore, $u^*$ is a stationary point of (\ref{our model}).
\end{proof}

\subsection{Proof of Theorem \ref{main thm 3}}\label{Appendix thm3}
The convergence of PGD with weakly convex prior function has been extensively studied. For details, one can refer to \cite{attouch2013convergence,hurault2024convergent}. For example, \cite{hurault2024convergent} prove that
\begin{theorem}[Theorem 1 by \cite{hurault2024convergent}]\label{thm hu}
    Let $F,H$ be proper, closed, bounded from below, and $H$ is differentiable with $L_H$-Lipschitz gradient, and $F$ is $r$-weakly convex. Then for $\tau<\max\{\frac{2}{L_H+r},\frac{1}{L_H}\}$, the iterates
    \begin{equation}
        u^{k+1}\in \operatorname{Prox}_{\tau F}\circ(\operatorname{I}-\tau\nabla H)(u^k)
    \end{equation}
    verify
    \begin{itemize}
        \item $(F(u^k)+H(u^k))$ is non-increasing and converges.
        \item All cluster points of the sequence $u^k$ are stationary points of $F+H$.
        \item If the sequence $u^k$ is bounded and if $F+H$ verifies the KL property at the cluster points of $u^k$, then $u^k$ converges to a stationary point of $F+H$.
    \end{itemize}
\end{theorem}

By Theorem \ref{thm hu}, we can prove Theorem \ref{main thm 3}.
\begin{proof}
Let $F$ be the proper, closed, weakly convex prior function. \rev{For the endpoint $t=1$ (or $\gamma=1$), the conjugate construction in Appendix~\ref{appendix general characterization}, applied to $(1-t)\operatorname{I}+t\operatorname{D}_\sigma$, still gives this prior with $r=\beta t(1-\gamma)/(\gamma+t(1-\gamma))$; smoothness of $F$ is not needed here.} Let $H={^1G}$ be the Moreau envelope of $G$, that is
\begin{equation}
    H(u)={^1G}(u):=\min\limits_v G(v)+\frac{1}{2}\|v-u\|^2.
\end{equation}
Then ${^1G}$ is diffentiable:
\begin{equation}
    \nabla {^1G}=\operatorname{I}-\operatorname{Prox}_{G}.
\end{equation}
Since $G$ is proper, closed, and convex, its proximal operator is firmly non-expansive. Therefore, $\nabla {^1G}$ is also firmly non-expansive, thus $1$-Lipschitz. \rev{By assumption, $F$ is coercive and $G$ is bounded below, so $H={^1G}\ge\inf G$ is bounded below. Thus $F+H$ is coercive; descent from Theorem~\ref{thm hu} bounds the iterates. The required KL property of $F+H$ is assumed directly.} By Theorem \ref{thm hu}, if $\frac{1}{\beta}<\max\{\frac{2}{1+r},1\}$, CoCo-PEGD converges to a stationary point of $F+{^1G}$.
\end{proof}

\subsection{Proximal operator on the fidelity for each task}\label{appendix proximal}

When $K\neq \operatorname{I}$, in general, there is no closed-form solution for $\operatorname{Prox}_{\frac{G}{\beta}}$. Given $f\in V$, we say $\hat{x}=\operatorname{Prox}_{\frac{G}{\beta}}(z)$, if 
\begin{equation}\label{tmp eq 1}
    \hat{x} =  \arg\min\limits_{x} \lambda\langle \textbf{1}, Kx-f\log Kx\rangle + \displaystyle \frac{\beta}{2}\|x-z\|^2.
\end{equation}

(\ref{tmp eq 1}) is solved by ADMM with $T$ iterations and $\rho>0$ as follows:
\begin{equation}\label{sub ADMM}
    \begin{array}{ll}
        x^{i+1} & =(\beta+\rho K^\top K)^{-1}(\beta z+\rho K^\top (y^i - w^i)), \vspace{0.5ex} \\
        y^{i+1} & = \arg\min\limits_{y} \lambda \langle \textbf{1}, y-f\log y\rangle + \displaystyle \frac{\rho}{2}\|Kx^{i+1}-y+w^i\|^2,\\
        w^{i+1} & = w^i + Kx^{i+1}-y^{i+1}.
    \end{array}
\end{equation}
When $K$ is a blur kernel, $x$-subproblem can be solved by the fast Fourier transformation as in \cite{pan2016l_0}. When $K=R$ is the Radon transform, one can calculate $(\beta+\rho K^\mathrm{T}K)^{-1}$ by the conjugate gradient method. There is a closed form solution for the $y$-subproblem by \cite{kumar2019low}:
\begin{equation}
    y^{i+1}=\displaystyle\frac{z^i+\sqrt{(z^i)^2+4\rho\lambda f}}{2\rho},
\end{equation}
where 
\begin{equation}
    z^i = \rho(Kx^{i+1}+\omega^i)-\lambda\textbf{1}.
\end{equation}
We initialize $x^0=y^0=z,w^0=0$, and output $\hat{x}=x^{T}$. \rev{We use $T=10$ inner iterations to approximate the proximal step.}

\subsection{Detailed Deconvolution results on each kernel}\label{appendix deblurring}
See Tables \ref{deblur color 100}-\ref{deblur color 50}.

\begin{figure}[htbp]
\begin{minipage}{0.11\linewidth}
  \centerline{\includegraphics[width=1.68cm]{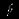}}
\end{minipage}
\hfill
\begin{minipage}{0.11\linewidth}
  \centerline{\includegraphics[width=1.68cm]{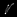}}
\end{minipage}
\hfill
\begin{minipage}{0.11\linewidth}
  \centerline{\includegraphics[width=1.68cm]{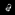}}
\end{minipage}
\hfill
\begin{minipage}{0.11\linewidth}
  \centerline{\includegraphics[width=1.68cm]{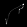}}
\end{minipage}
\hfill
\begin{minipage}{0.11\linewidth}
  \centerline{\includegraphics[width=1.68cm]{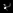}}
\end{minipage}
\hfill
\begin{minipage}{0.11\linewidth}
  \centerline{\includegraphics[width=1.68cm]{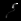}}
\end{minipage}
\hfill
\begin{minipage}{0.11\linewidth}
  \centerline{\includegraphics[width=1.68cm]{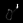}}
\end{minipage}
\hfill
\begin{minipage}{0.11\linewidth}
  \centerline{\includegraphics[width=1.68cm]{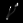}}
\end{minipage}
\centering\caption{Eight blur kernels from \cite{levin2009understanding}.}
\label{kernel}
\end{figure}

\begin{table*}[htbp]
\caption{Average Deconvolution PSNR and SSIM performance by different methods on CBSD68 dataset with Poisson noise with peak value $p=100$.}\label{deblur color 100}
\begin{center}
\resizebox{1\textwidth}{!}{
\begin{tabular}{cccccccccc}
\toprule
p=100 & kernel1         & kernel2         & kernel3         & kernel4         & kernel5         & kernel6         & kernel7         & kernel8         & Average         \\
\midrule
DPIR                          & 26.24           & 25.97           & {26.66}  & 25.65           & {27.61}  & {27.42}  & 26.52           & 26.00           & 26.51           \\
                              & \textbf{0.7321} & \textbf{0.7209} & \textbf{0.7439} & {0.7016} & \textbf{0.7870} & \textbf{0.7780} & \textbf{0.7462} & \textbf{0.7257} & \textbf{0.7419} \\
RMMO-DRS                           & 25.60           & 25.39           & 26.07           & 25.19           & 26.92           & 26.59           & 26.06           & 25.67           & 25.94           \\
                              & 0.6895          & 0.6804          & 0.7057          & 0.6665          & 0.7426          & 0.7296          & 0.7089          & 0.6921          & 0.7019          \\
Prox-DRS & 25.32 & 25.03 & 25.82 & 24.80 & 26.75 & 26.39 & 25.94& 25.40&25.68\\
& 0.6546 & 0.6488 & 0.6842 & 0.6336 & 0.7256 & 0.7050 &0.6924 & 0.6668&0.6764\\

DPS & 23.40 & 23.23 & 24.19 & 22.86 & 24.74 & 24.06 & 23.60 & 23.10 & 23.65 \\
& 0.5941 & 0.5882 & 0.6303 & 0.5699 & 0.6550 & 0.6265 & 0.6039 & 0.5815 & 0.6062\\

DiffPIR & 24.46 & 24.45 & 25.03 & 24.17 & 25.66 & 25.27 & 24.93 & 24.56 & 24.82\\
& 0.6250 & 0.6251 & 0.6510 & 0.6085 & 0.6838 & 0.6664 & 0.6509 & 0.6325 & 0.6429\\

SNORE & 25.75 & 25.99 & 26.56 & 25.70 & 27.13 & 26.85 & 26.55 & 26.09&26.33 \\
 & 0.6985 & 0.6993 & 0.7202 & 0.6856 & 0.7485 & 0.7405 & 0.7249 & 0.7085 &0.7158\\
PnPI-HQS                       & 26.25           & 25.95           & 26.32           & 25.70           & 27.32           & 27.15           & 26.55           & 26.15           & 26.42           \\
                              & 0.7024          & 0.6930          & 0.7124          & 0.6797          & 0.7570          & 0.7466          & 0.7269          & 0.7084          & 0.7158          \\
B-RED & 23.89  & 23.55  & 24.09  & 23.33  & 25.04  & 24.61  & 24.33  & 23.88  & 24.09  \\
                              & 0.6206          & 0.6188          & 0.6537          & 0.5988          & 0.6824          & 0.6584          & 0.6434          & 0.6200          & 0.6370          \\

CoCo-ADMM& \textbf{26.66} & \textbf{26.46} & \textbf{26.93} & \textbf{26.23} & \textbf{27.78} & \textbf{27.52} & \textbf{26.96} & \textbf{26.59} &\textbf{26.89} \\
                              &  0.7270&0.7183&0.7343&\textbf{0.7075} & 0.7707 & 0.7630 & 0.7404 & 0.7254&0.7358 \\
                              \bottomrule
\end{tabular}
}
\end{center}
\end{table*}

\begin{table}[htbp]
\caption{Average Poisson deblurring PSNR and SSIM performance by different methods on CBSD68 dataset with Poisson noise with peak value $p=50$.}\label{deblur color 50}
\begin{center}
\resizebox{1\textwidth}{!}{
\begin{tabular}{cccccccccc}
\toprule
p=50 & kernel1         & kernel2         & kernel3         & kernel4         & kernel5         & kernel6         & kernel7         & kernel8         & Average         \\
\midrule
DPIR                         & 25.10           & 24.83           & 25.66  & 24.50           & 26.45  & 26.20  & 25.41           & 24.89           & 25.38           \\
                             & 0.6773 & 0.6656 & {0.6988} & 0.6480 & {0.7390} & {0.7261} & 0.6981 & 0.6752 & 0.6910 \\
RMMO-DRS                          & 24.86           & 24.73           & 25.39           & 24.51           & 25.84           & 25.52           & 25.12           & 24.83           & 25.10           \\
                             & 0.6439          & 0.6419          & 0.6690          & 0.6292          & 0.6826          & 0.6697          & 0.6549          & 0.6457          & 0.6546          \\
Prox-DRS & 24.89 & 24.69 &  25.45 & 24.43 & 26.14 & 25.81 & 25.36 &24.92& 25.21\\ 
& 0.6363 & 0.6319 & 0.6676 & 0.6164 & 0.7024 & 0.6829 & 0.6714 & 0.6484& 0.6572\\
DPS & 22.92 & 22.76 & 23.72 & 22.34 & 24.13 & 23.48 & 22.98 & 22.49 & 23.10 \\
& 0.5724 & 0.5668 & 0.6088 & 0.5475 & 0.6266 & 0.5990 & 0.5758 & 0.5561 & 0.5816 \\

DiffPIR & 23.79 & 23.73 & 24.43 & 23.45 & 24.91 & 24.45 & 24.16 & 23.77 & 24.08\\
& 0.5915 & 0.5894 & 0.6220 & 0.5744 & 0.6478 & 0.6263 & 0.6134 & 0.5941 & 0.6074\\

SNORE & 24.57 & 25.06 & 25.80 & 24.75 & 25.78 & 25.34& 25.34&25.01& 25.21\\
& 0.6522 & 0.6617 & 0.6893 & 0.6490 & 0.6967 & 0.6863 & 0.6749 & 0.6648 & 0.6719\\
PnPI-HQS                     & {25.54}  & 25.29           & 25.80           & 25.02           & 25.60           & {26.37}  & 25.83           & 25.41           & 25.61           \\
                             & \textbf{0.6972} & 0.6705 & 0.6943          & {0.6563} & 0.7359          & 0.7231          & {0.7039} & {0.6837} & {0.6956} \\
B-RED & 23.49  & 23.41  & 23.86  & 22.99  & 24.69  & 24.49  & 23.86  & 23.46  & 23.78  \\
                             & 0.6083          & 0.6059          & 0.6436          & 0.5854          & 0.6692          & 0.6441          & 0.6264          & 0.6027          & 0.6232          \\
CoCo-ADMM                  & \textbf{27.74}  & \textbf{25.57}  & \textbf{26.16}          & \textbf{25.32}  &     \textbf{26.89}       &  \textbf{26.52}          &  \textbf{26.09} & \textbf{25.67}  & \textbf{26.00}  \\
                              & {0.6906}          &         \textbf{0.6817}  & \textbf{0.7072}          &        \textbf{0.6701}   & \textbf{0.7412}          &         \textbf{0.7289}  &  \textbf{0.7098}         &\textbf{0.6911}           & \textbf{0.7026}  \\
                              
            \bottomrule
\end{tabular}
}
\end{center}
\end{table}

\subsection{Single photon imaging}\label{appendix single}

 We test the proposed methods in a low-light real-world scenario: single photon imaging task by a time-correlated single-photon avalanche diode (SPAD) camera \cite{shin2016photon}. In this task, a SPAD array is used to track periodic light pulses in flight, and each detector only recieves about $1$ signal photon per pixel on average ($p\approx 1$). By the uncertainty principle \cite{heisenberg1927anschaulichen,robertson1929uncertainty}, since the momentum of the photons is known, the position of the photons cannot be accurately recorded. Therefore, the arrival time of each photon in the SPAD array is a random variable, and can be modelled by a Poisson process. By the filtered histogram method, a reflectivity image\footnote{The data is kindly provided by \cite{shin2016photon} in \href{https://github.com/photon-efficient-imaging/single-photon-camera}{github.com/photon-efficient-imaging/single-photon-camera}.} is obtained, see Fig. \ref{fig single} (a).

We compare the visual denoising performance on Fig. \ref{fig single} (a) by some state-of-the-art Poisson noise removal methods: Photon TV using a TV prior specialized for this problem \cite{shin2016photon}; BM3D-VST, a BM3D method with variance stabilization transform designed for Poisson noises \cite{azzari2016variance}; DnCNN \cite{zhang2017beyond} trained as a Poisson denoiser; VDIR, a variational inference network \cite{soh2022variational}; VBDNet, a variational bayesian deep network for blind poisson denoising \cite{liang2023variational}. We also compare three PnP methods, DPIR, RMMO-DRS, and PnPI-HQS, as introduced before. 

We show the visual results in Fig. \ref{fig single}. It can be seen in Figs. \ref{fig single} (b)-(c) that, non-deep methods provide blurred results with unclear edges. End-to-end deep learning methods provide over-smoothed results, and cannot restore the fine details in the flower, see Figs. \ref{fig single} (d)-(f). Compared with other PnP methods, CoCo methods can recover the fine textures in the flower, as well as the sharp edges in the equation, see Figs. \ref{fig single} (g)-(k). Since there is no reference image in this real-world scenario, we report the relative error curves \rev{as empirical evidence of stabilization} in Fig. \ref{fig single} (l).

\subsection{Ablation study on $\gamma$ and $t$ in CoCo-PnP}\label{appendix diff}

We report the PSNR values by CoCo-PnP with different $\gamma$ and $t$ in the deblurring task in Table \ref{tab: diff t gamma}. It can be seen that, the proposed methods are sensitive to $\gamma$, because it determines the denoising performance, but not sensitive to $t$. 

\begin{table*}[htbp]
\caption{PSNR results by CoCo-PnP with different $\gamma$ and $t$ when deblurring the image `0005' from CBSD68 and kernel 8 from Levin's dataset. The Poisson noise level is $50$. \rev{We run the algorithm for 500 iterations.} When $\gamma=0.25$, $t_0(\gamma)=1/3$, thus $t<1/3$. When $\gamma=0.5$, $t_0(\gamma)\approx 0.3761$, thus \rev{$t<t_0(0.5)\approx0.3761$}.}\label{tab: diff t gamma}
\begin{center}
\resizebox{0.6\textwidth}{!}{
\begin{tabular}{ccccccc}
\toprule
$\gamma=0.25$&$t$ & $0.3333$ & $0.300$& $0.200$ &$0.100$&$0.001$ \\
\midrule
CoCo-ADMM &PSNR & 25.33 & 25.31 & 25.30 & 25.24 & 25.16 \\
\midrule
CoCo-PEGD & PSNR& 25.32 & 25.30 & 25.27 & 25.22 & 24.98\\
\midrule
$\gamma=0.50$&$t$ & $0.3761$ & $0.300$& $0.200$ &$0.100$&$0.001$ \\
\midrule
CoCo-ADMM &PSNR & 24.88&24.87& 24.79& 24.64& 24.47 \\
\midrule
CoCo-PEGD & PSNR& 24.84 & 24.82 & 24.71 & 24.54 & 24.35\\
\bottomrule
\end{tabular}
}
\end{center}
\end{table*}

\subsection{Low-dose CT reconstruction}\label{appendix CT}

In order to further illustrate the effectiveness of the proposed methods, we consider the sparse-view low-dose CT reconstruction task. In this task, $K=R$ is the Radon transform, and $f$ in Eq. (\ref{model0}) is the down-sampled data in the Radon field. The Radon field data is corrupted by the Poisson noise.

For the test set, we select ten typical images from ``the
 NIH-AAPM-Mayo Clinic Low Dose CT Grand Challenge'' \cite{Mayo}, see Fig. \ref{fig CT test}.

For the comparison methods, we use: FBP, the conventional filtered back projection method; PWLS-TGV, penalized weighted least-squares (PWLS) with total generalized variation (TGV) prior \cite{niu2014sparse}; PWLS-CSCGR, PWLS with convolutional sparse coding and gradient regularization \cite{bao2019convolutional}; UNet, which take the result by FBP as the input \cite{jin2017deep}; WNet, which uses two UNets to construct the data from both the image and Radon domain \cite{cheslerean2023wnet}. 

We show the reconstruction results with $60$ projection views and Poisson noises with peak value $p=100,500$ in Table \ref{tab deblur}. It can be seen that, compared with two iterative PWLS-based methods, and two end-to-end deep learning methods, the proposed CoCo methods have a significant improvement in PSNR and SSIM values.


\begin{figure}[htbp]
\begin{minipage}{0.19\linewidth}
  \centerline{\includegraphics[width=2.83cm]{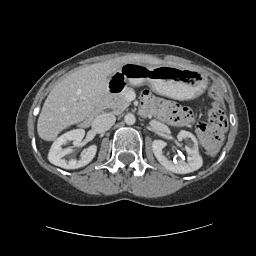}}
\end{minipage}
\hfill
\begin{minipage}{0.19\linewidth}
  \centerline{\includegraphics[width=2.83cm]{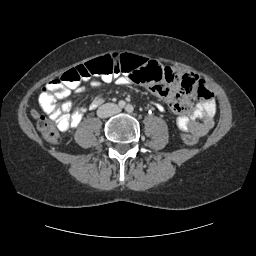}}
\end{minipage}
\hfill
\begin{minipage}{0.19\linewidth}
  \centerline{\includegraphics[width=2.83cm]{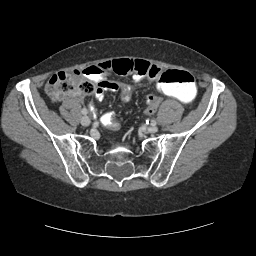}}
\end{minipage}
\hfill
\begin{minipage}{0.19\linewidth}
  \centerline{\includegraphics[width=2.83cm]{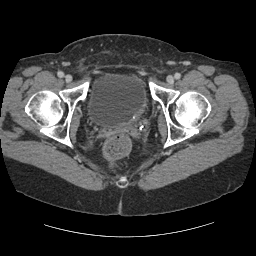}}
\end{minipage}
\hfill
\begin{minipage}{0.19\linewidth}
  \centerline{\includegraphics[width=2.83cm]{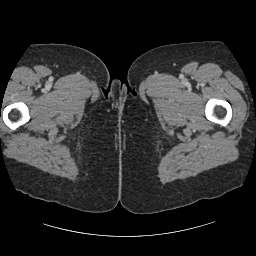}}
\end{minipage}
\hfill
\begin{minipage}{0.19\linewidth}
  \centerline{\includegraphics[width=2.83cm]{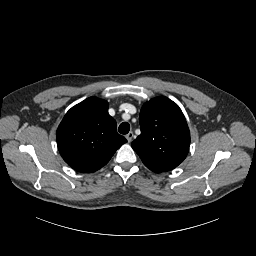}}
\end{minipage}
\hfill
\begin{minipage}{0.19\linewidth}
  \centerline{\includegraphics[width=2.83cm]{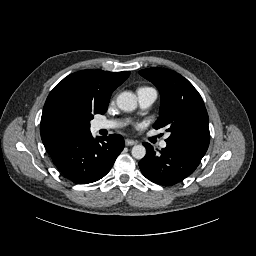}}
\end{minipage}
\hfill
\begin{minipage}{0.19\linewidth}
  \centerline{\includegraphics[width=2.83cm]{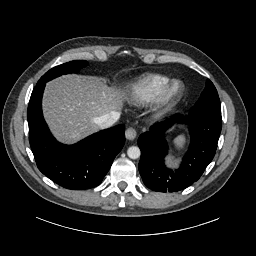}}
\end{minipage}
\hfill
\begin{minipage}{0.19\linewidth}
  \centerline{\includegraphics[width=2.83cm]{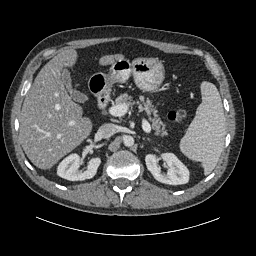}}
\end{minipage}
\hfill
\begin{minipage}{0.19\linewidth}
  \centerline{\includegraphics[width=2.83cm]{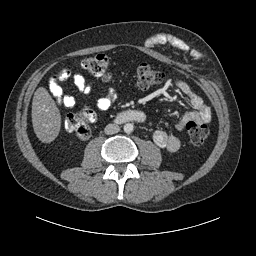}}
\end{minipage}
\centering\caption{CT test images.}
\label{fig CT test}
\end{figure}

\begin{figure*}[htbp]

\begin{minipage}{0.19\linewidth}
  \centerline{\includegraphics[width=2.83cm]{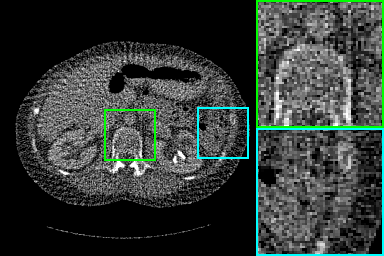}}
  \centerline{(a) FBP}
\end{minipage}
\hfill
\begin{minipage}{0.19\linewidth}
  \centerline{\includegraphics[width=2.83cm]{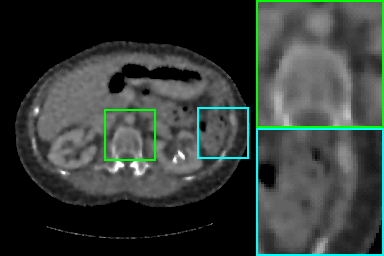}}
  \centerline{(b) PWLS-TGV}
\end{minipage}
\hfill
\begin{minipage}{0.19\linewidth}
  \centerline{\includegraphics[width=2.83cm]{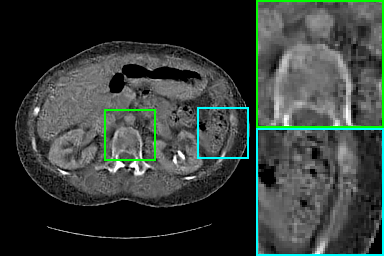}}
  \centerline{(c) PWLS-CSCGR}
\end{minipage}
\hfill
\begin{minipage}{0.19\linewidth}
  \centerline{\includegraphics[width=2.83cm]{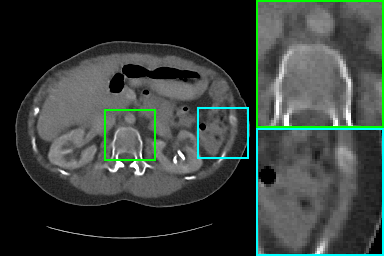}}
  \centerline{(d) UNet}
\end{minipage}
\hfill
\begin{minipage}{0.19\linewidth}
  \centerline{\includegraphics[width=2.83cm]{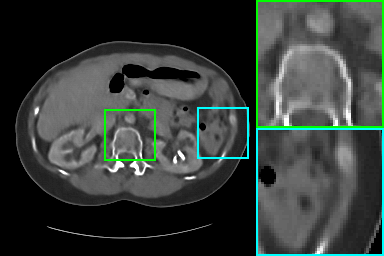}}
  \centerline{(e) WNet}
\end{minipage}
\hfill
\begin{minipage}{0.19\linewidth}
  \centerline{\includegraphics[width=2.83cm]{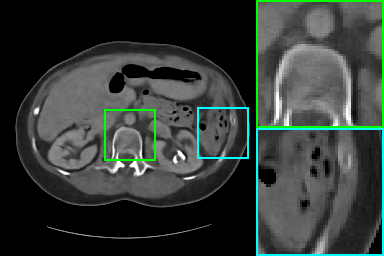}}
  \centerline{(f) CoCo-ADMM}
\end{minipage}
\hfill
\begin{minipage}{0.19\linewidth}
  \centerline{\includegraphics[width=2.83cm]{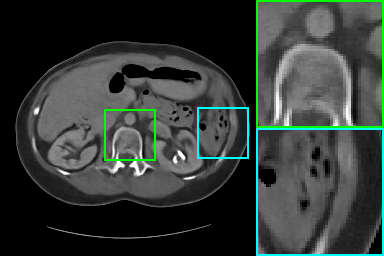}}
  \centerline{(g) CoCo-PEGD}
\end{minipage}
\hfill
\begin{minipage}{0.19\linewidth}
  \centerline{\includegraphics[width=2.83cm]{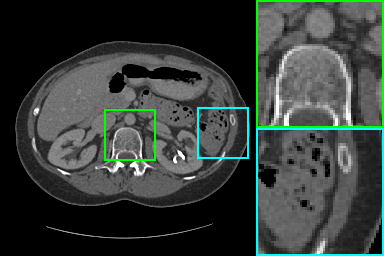}}
  \centerline{(h) High-dose reference}
\end{minipage}
\hfill
\begin{minipage}{0.19\linewidth}
  \centerline{\includegraphics[width=2.83cm]{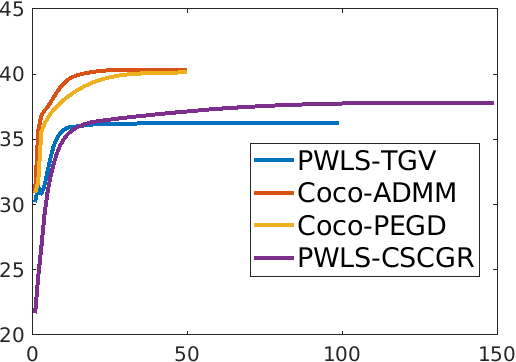}}
  \centerline{(i) PSNR}
\end{minipage}
\hfill
\begin{minipage}{0.19\linewidth}
  \centerline{\includegraphics[width=2.83cm]{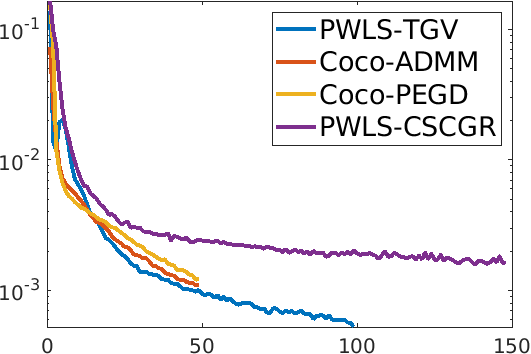}}
  \centerline{(j) Relative error}
\end{minipage}
\centering\caption{Sparse view CT reconstruction results with $60$ views and Poisson noises (peak=$500$). (a) FBP, PSNR=$30.91$dB. (b) PWLS-TGV, PSNR=$36.12$dB. (c) PWLS-CSCGR, PSNR=$37.80$dB. (d) UNet, PSNR=$38.54$dB. (e) WNet, PSNR=$38.74$dB. (f) CoCo-ADMM, PSNR=$40.46$dB. (g) CoCo-PEGD, PSNR=$40.28$dB. (h) High-dose reference. (i) PSNR curves. (j) Relative error curves.}
\label{fig CT}
\end{figure*}

\subsection{Poisson denoising results}\label{appendix poisson denoising}

We apply CoCo-ADMM and CoCo-PEGD to Poisson denoising problems. In this task, $K=\operatorname{I}$ is the identity. The average PSNR and SSIM values on CBSD68 are summarized in Table \ref{tab denoise}. It can be seen that compared with other convergent methods, the proposed methods achieve the best PSNR and SSIM values. It validates the effectiveness of CoCo methods.

\begin{table}[htbp]
\caption{Average denoising PSNR and SSIM performance by different methods on CBSD68 with peak value $p=30$ and $p=20$.}\label{tab denoise}
\begin{center}
\resizebox{0.5\textwidth}{!}{
\begin{tabular}{ccccc}
\toprule
& \multicolumn{2}{c}{$p=30$} & \multicolumn{2}{c}{$p=20$} \\
\midrule
 & PSNR & SSIM & PSNR & SSIM \\
\midrule 
DPIR        & 29.85 & \textbf{0.8654} & 28.68 & {0.8300}  \\
RMMO-DRS         & 27.90 & 0.8008 & 27.74 & 0.7887  \\
Prox-DRS         & 28.99 & 0.8087 & 27.94 & 0.7724  \\
SNORE         & 29.47 & 0.8474 & 28.34 & 0.8226  \\
PnPI-HQS     & 29.90 & 0.8613 & 28.63 & {0.8156}  \\ 
B-RED       & 28.80 & 0.8041 & 26.26 & 0.7362  \\
CoCo-ADMM & \textbf{29.99} & 0.8604 & \underline{28.76} & \underline{0.8329}\\
CoCo-PEGD & \textbf{29.99} & \underline{0.8625} & \textbf{28.77} & \textbf{0.8332}  \\
\bottomrule
\end{tabular}
}
\end{center}
\end{table}

\subsection{Computational cost}\label{sec:cputime}

In Table \ref{tab:cputime}, we give the average computation time in seconds, as well as the memory cost in MB by different methods when deblurring a $256\times 256$ image. It can be seen that the proposed methods are the most efficient with least memory cost.

\begin{table}[htbp]
	\centering
	\begin{tabular}{cccccccc}
		\hline
		\quad &B-RED& Prox-DRS& SNORE& DiffPIR& DPS & CoCo-ADMM & CoCo-PEGD \\
		\hline
		Iteration & 300&100 &200 & 100 & 1000 & 50 & 100\\ 
		\hline
		Time & 42.5s& 13.6s &74.1s& 11.8s & 177s &5.24s & 9.08s\\ 
		\hline
		Memory &4946MB &3982MB&3624MB&2998MB&4694MB&2304MB&2304MB\\
		\hline
	\end{tabular}
	\caption{Average computation time in seconds, and memory cost in MB by different methods when deblurring a $256\times 256$ image.}
	\label{tab:cputime}
\end{table}

\subsection{Performances under extreme conditions}\label{sec:extreme}
In Fig. \ref{fig c05}, we show an example when deblurring an image under severe noisy condition. It can be seen in Fig. \ref{fig c05} (a) that the image is severely corrupted. We compare results by different methods. It can be seen that CoCo-ADMM and CoCo-PEGD provide sharper edges. \rev{The finite-run curves show empirical stabilization in this case, see Figs.~\ref{fig c05} (m)--(n).}
\begin{figure}[htbp]

\begin{minipage}{0.13\linewidth}
  \centerline{\includegraphics[width=1.99cm]{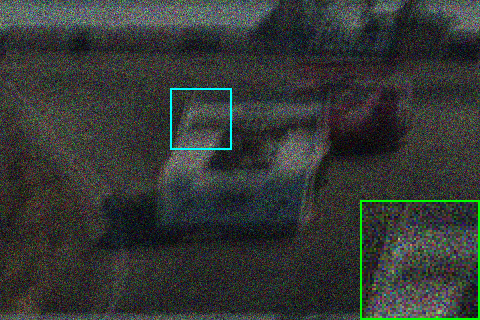}}
  \centerline{\small (a) Blur}
\end{minipage}
\hfill
\begin{minipage}{0.13\linewidth}
  \centerline{\includegraphics[width=1.99cm]{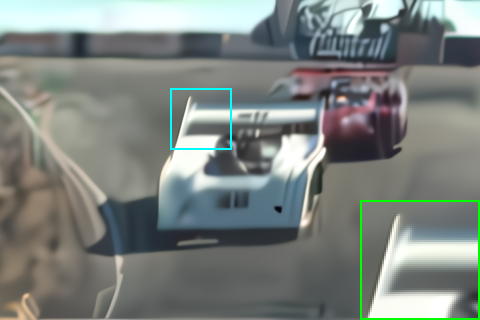}}
  \centerline{\small (b) DPIR}
\end{minipage}
\hfill
\begin{minipage}{0.13\linewidth}
  \centerline{\includegraphics[width=1.99cm]{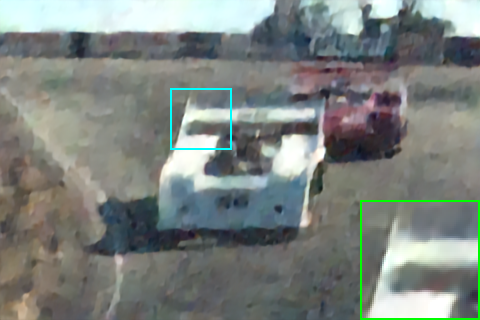}}
  \centerline{\small (c) RMMO-DRS}
\end{minipage}
\hfill
\begin{minipage}{0.13\linewidth}
  \centerline{\includegraphics[width=1.99cm]{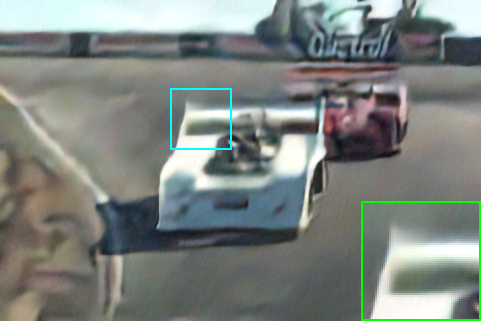}}
  \centerline{\small (d) Prox-DRS}
\end{minipage}
\hfill
\begin{minipage}{0.13\linewidth}
  \centerline{\includegraphics[width=1.99cm]{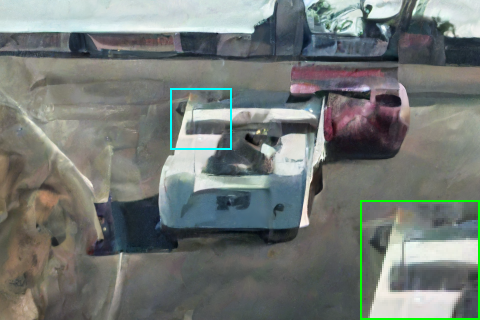}}
  \centerline{\small (e) DPS}
\end{minipage}
\hfill
\begin{minipage}{0.13\linewidth}
  \centerline{\includegraphics[width=1.99cm]{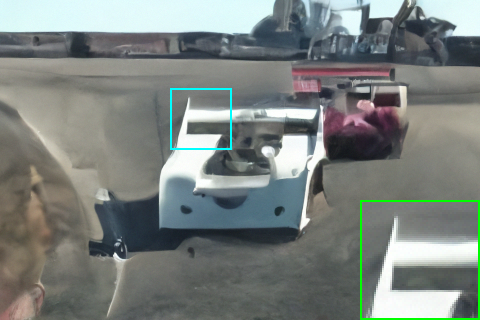}}
  \centerline{\small (f) DiffPIR}
\end{minipage}
\hfill
\begin{minipage}{0.13\linewidth}
  \centerline{\includegraphics[width=1.99cm]{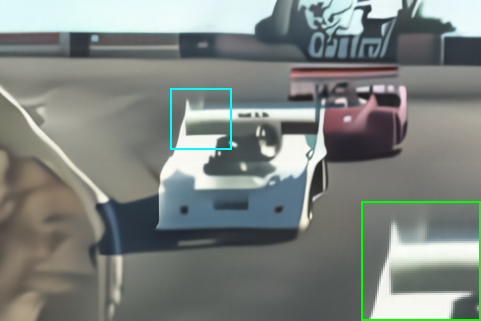}}
  \centerline{\small (g) SNORE}
\end{minipage}
\hfill
\begin{minipage}{0.13\linewidth}
  \centerline{\includegraphics[width=1.99cm]{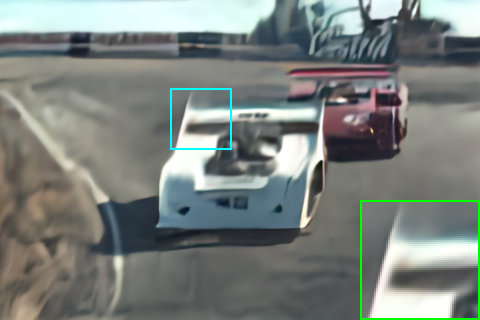}}
  \centerline{\small (h) PnPI-HQS}
\end{minipage}
\hfill
\begin{minipage}{0.13\linewidth}
  \centerline{\includegraphics[width=1.99cm]{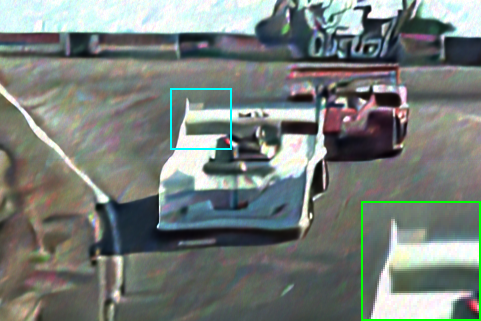}}
  \centerline{\small (i) B-RED\ \ }
\end{minipage}
\hfill
\begin{minipage}{0.13\linewidth}
  \centerline{\includegraphics[width=1.99cm]{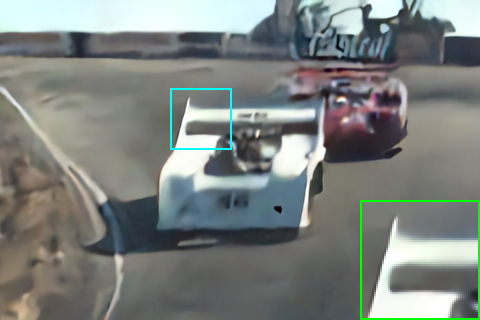}}
  \centerline{\small (j) CoCo-ADMM (k)}
\end{minipage}
\hfill
\begin{minipage}{0.13\linewidth}
  \centerline{\includegraphics[width=1.99cm]{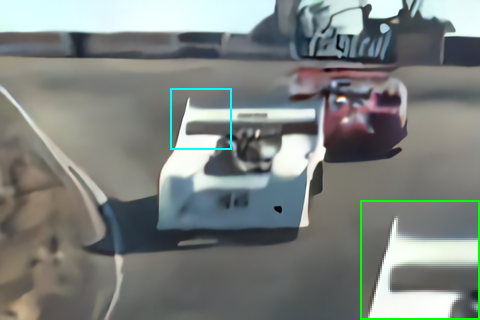}}
  \centerline{\small\ \ \ \ \  CoCo-PEGD}
\end{minipage}
\hfill
\begin{minipage}{0.13\linewidth}
  \centerline{\includegraphics[width=1.99cm]{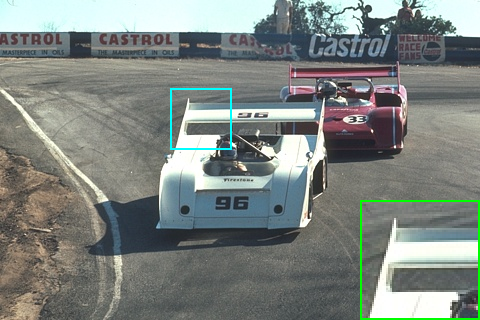}}
  \centerline{\small (l) Clean}
\end{minipage}
\hfill
\begin{minipage}{0.13\linewidth}
  \centerline{\includegraphics[width=1.99cm]{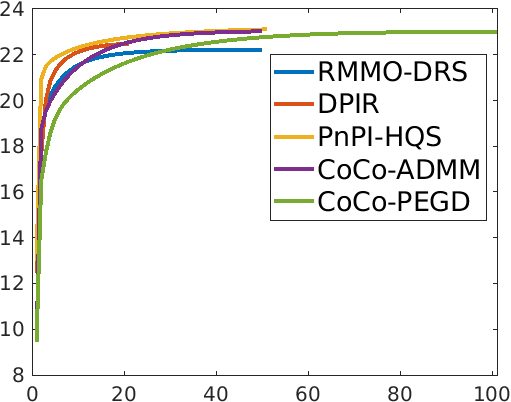}}
  \centerline{\small (m) PSNR}
\end{minipage}
\hfill
\begin{minipage}{0.13\linewidth}
  \centerline{\includegraphics[width=1.99cm]{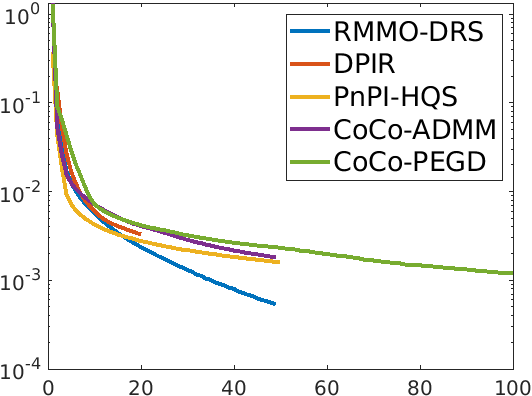}}
  \centerline{\small (n) Relative error}
\end{minipage}
\centering\caption{Deconvolution results by different methods on the image `0005' from CBSD68 with kernel $8$ and $p=10$ Poisson noises. (a) Blur image. (b) DPIR, PSNR=$22.49$dB. (c) RMMO-DRS, PSNR=$22.23$dB. (d) Prox-DRS, PSNR=$22.54$dB. (e) DPS, PSNR=$ 20.70$dB. (f) DiffPIR, PSNR=$21.66$dB. (g) SNORE, PSNR=$23.31$dB. (h) PnPI-HQS, PSNR=$23.14$dB. (i) B-RED, PSNR=$21.15$dB. (j) CoCo-ADMM, PSNR=$23.04$dB. (k) CoCo-PEGD, PSNR=$23.02$dB. (l) Clean image. (m) PSNR curves. (n) Relative error curves. $x$-axis denotes the iteration number.}
\label{fig c05}
\end{figure}

\subsection{Ablation study on the parameters $\alpha_1$ and $\alpha_2$}\label{sec:ablation}
\begin{equation}\label{loss:ablation}
    Loss(\theta)=\mathbb{E}\|\operatorname{D}_\sigma(x+\xi;\theta)-x\|_1+\alpha_1\|\operatorname{J}-\operatorname{J}^\top\|_* +\alpha_2 \max\{\|2\gamma\operatorname{J}-\operatorname{I}\|_*,1-\epsilon\}.
\end{equation}

In the loss function (\ref{loss:ablation}), $\alpha_1$ controls the penalty strength of the Hamiltonian term to encourage conservativeness, and $\alpha_2$ controls the penalty strength of the spectral term to encourage cocoerciveness.

As shown in Table \ref{tab validation}, DRUNet without regularization terms has an expansive residual part. When $\alpha_1$ gets bigger, the mean symmetry error gets lower. When $\alpha_2$ is smaller than $1$e-2, the regularized denoiser may not be cocoercive. 

Therefore, we set the penalty parameters large enough to encourage the properties we want. However, when $\alpha$ gets too large, the denoising performance may be sacrificed. We empirically set $\alpha_1=1$, $\alpha_2=1$e-2 in the experiments.

\begin{table}[htbp]
\caption{Mean symmetry error $\|\operatorname{J}-\operatorname{J}^\top\|_* $ with $N=1$ and maximal \rev{estimates} of the norm $\|2\gamma\operatorname{J}-\operatorname{I}\|_*$ with $N=30$ on CBSD68 for various noise levels $\sigma$ and $\gamma=0.50,0.25$.}\label{tab validation}
\begin{center}
\resizebox{0.8\textwidth}{!}{
\begin{tabular}{ccccccc}
\toprule
             & 15     & 25     & 40     &  Norms & $\alpha_1$ & $\alpha_2$ \\
\midrule
DRUNet              &  $4.1$e+0& $4.2$e+1 & $1.8$e+1  & $\|\operatorname{J}-\operatorname{J}^\top\|_* $ & 0 & 0 \\
$0.50$-CoCo-DRUNet  & $8.6$e-3 & $2.5$e-4 & $7.1$e-4 & $\|\operatorname{J}-\operatorname{J}^\top\|_* $ & 1 & $1$e-2  \\
$0.50$-CoCo-DRUNet  & $2.9$e-1 & $2.6$e-2 & $9.5$e-2 & $\|\operatorname{J}-\operatorname{J}^\top\|_* $ & $1$e-2 & $1$e-2  \\
$0.25$-CoCo-DRUNet  & $3.1$e-4 & $1.8$e-4 & $3.9$e-4 & $\|\operatorname{J}-\operatorname{J}^\top\|_* $ & 1 & $1$e-2  \\
$0.25$-CoCo-DRUNet   & $5.9$e-1 & $8.5$e-1 & $7.9$e-1 & $\|\operatorname{J}-\operatorname{J}^\top\|_* $ & $1$e-2 & $1$e-2   \\
\midrule
DRUNet              & $3.285$ &$4.343$  & $6.283$ & $\|\operatorname{J}-\operatorname{I}\|_*$ & 0 & 0\\
$0.50$-CoCo-DRUNet  & $0.994$ & $0.992$ &$0.972$  &  $\|\operatorname{J}-\operatorname{I}\|_*$ & 1 & $1$e-2 \\
$0.50$-CoCo-DRUNet  & $1.011$ & $1.244$ & $1.500$ &  $\|\operatorname{J}-\operatorname{I}\|_*$ & 1 & $1$e-3 \\
$0.25$-CoCo-DRUNet  & $0.986$ & $0.969$ &$0.982$  &  $\|0.5\operatorname{J}-\operatorname{I}\|_*$ & 1 & $1$e-2 \\
$0.25$-CoCo-DRUNet  & $0.992$ & $1.010$ & $1.211$  &  $\|0.5\operatorname{J}-\operatorname{I}\|_*$ & 1 & $1$e-3 \\
\bottomrule
\end{tabular}
}
\end{center}
\end{table}
\begin{figure}[htbp]
  \centerline{\includegraphics[width=14.00cm]{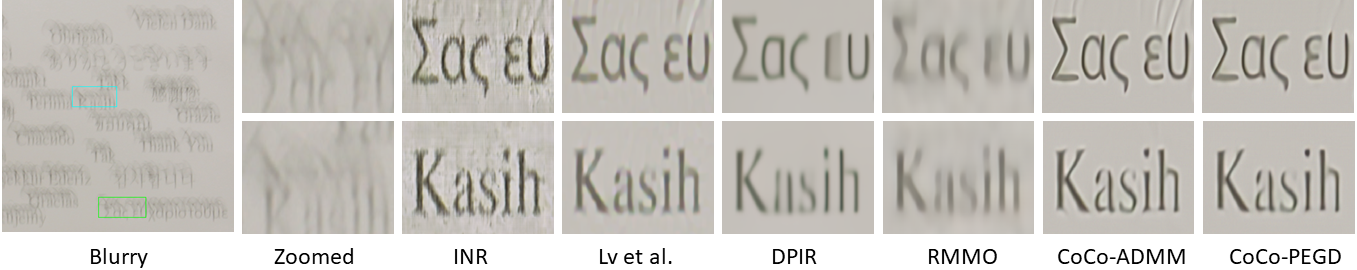}}
\centering\caption{Blind decovolution results by different methods.}
\label{fig:real-blur}
\end{figure}

\begin{table}[htbp]
    \centering
    \begin{tabular}{ccccccc}
    \hline
    \quad & INR  & Lv et al. & DPIR & RMMO & CoCo-ADMM & CoCo-PEGD\\
    \hline
    Average & 0.2990 & 0.3444 & 0.3629 & 0.2664 & 0.5461 & \textbf{0.5474}\\ 
    \hline
    \end{tabular}
    \caption{Average CLIP-IQA results across all deblurred images}
    \label{tab:real-blur}
\end{table}
\subsection{Blind Deblur Results}\label{sec:deblur}

In this section, we evaluate the performance of our proposed methods on real-world blind deblurring. We select 18 images from the dataset by Lai et al. \cite{Lai-16}, which captures challenging real-world blurry scenes. Four state-of-the-art methods are benchmarked: Deblur-INR \cite{zhang2024cross}, Lv et al. \cite{Liu-21}, DPIR \cite{zhu2023denoising}, and RMMO \cite{terris2020building}. Deblur-INR is a self-supervised method that jointly estimates clean images and blur kernels, while Lv et al. \cite{Liu-21} applies total variation regularization to refine latent image estimation in a blind deblurring framework. In contrast, DPIR, RMMO, and our proposed CoCo-ADMM/CoCo-PEGD require pre-estimated blur kernels as input. To ensure fairness, all kernel-dependent methods utilize blur kernels generated by the blind deblurring framework of Liu et al. \cite{Liu-21}. Visual comparisons in Fig. \ref{fig:real-blur} demonstrate that CoCo-ADMM and CoCo-PEGD produce sharper textures and fewer artifacts than baseline methods.

For quantitative evaluation, we employ the CLIP-IQA metric \cite{wang2023exploring}, a non-reference image quality assessment tool that leverages the pretrained vision-language representation of CLIP (Contrastive Language–Image Pretraining) to evaluate perceptual quality without requiring pristine reference images. The average CLIP-IQA scores across all deblurred images are summarized in Table \ref{tab:real-blur}. Despite relying on the same estimated kernels as DPIR and RMMO, our methods achieve superior performance, highlighting their robustness to kernel estimation errors and enhanced restoration capability.

\subsection{Blind Denoise Results}\label{sec:denoise}
\begin{figure}[htbp]
  \centerline{\includegraphics[width=14.00cm]{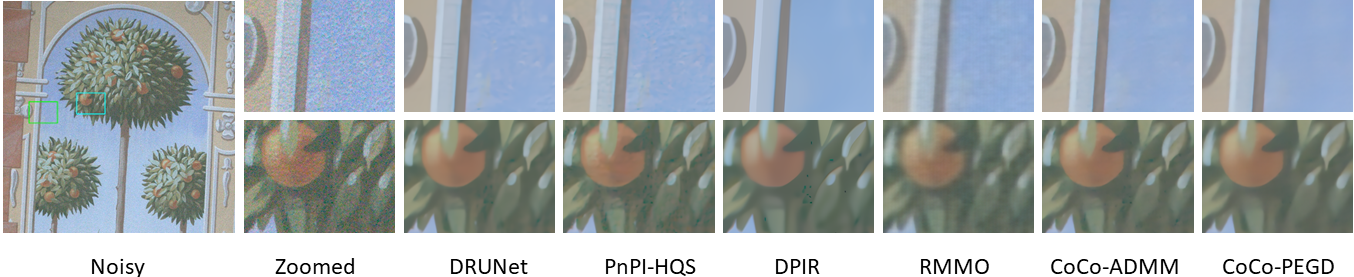}}
\centering\caption{Blind noise removal results by different methods.}
\label{fig:real-denoise}
\end{figure}
\begin{figure}[htbp]
  \centerline{\includegraphics[width=14.00cm]{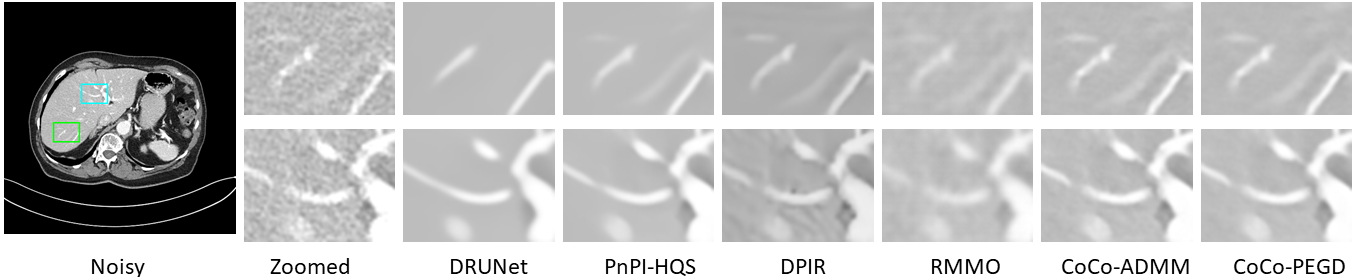}}
\centering\caption{Blind noise suppression results in CT images by different methods.}
\label{fig:real-ct}
\end{figure}
In this section, we evaluate the performance of our proposed denoising framework on two benchmark datasets that capture real-world noise characteristics. For medical imaging, we employ the Low-Dose CT dataset from Mayo Clinic \cite{Moen-21}, which provides paired low-dose CT scans with Poisson noise.  

For natural images, we utilize the DND dataset \cite{dnd_dataset}, comprising 50 real noisy photographs corrupted by a mixture of Poisson and Gaussian noise. Following the standard evaluation protocol, perceptual quality is quantified using the CLIP-IQA metric \cite{wang2023exploring}.

To ensure a fair comparison, we benchmark our proposed methods (CoCo-ADMM and CoCo-PEGD) against five representative techniques: DRUNet\cite{zhu2023denoising}, DPIR \cite{zhu2023denoising} , RMMO \cite{terris2020building} , and PnPI-HQS \cite{AHQS}. Quantitative results in Table~\ref{tab:color-denoise} demonstrate that our proposed methods achieve state-of-the-art performance across both imaging domains.

\begin{table}[htbp]
	\centering
	\begin{tabular}{ccccccc}
		\hline
		\quad & DRUNet & DPIR & RMMO & PnPI-HQS & CoCo-ADMM & CoCo-PEGD \\
		\hline
		Color images & 0.4464 & 0.4509 & 0.4561 & 0.4586 & \textbf{0.4596} & 0.4593 \\ 
		\hline
		CT & 0.4154 & 0.4258 & 0.4398 & 0.4011 & 0.4707 & \textbf{0.4730} \\ \hline
	\end{tabular}
	\caption{Average CLIP-IQA results for blind image denoising.}
	\label{tab:color-denoise}
\end{table}


\section{Revision note}\label{appendix:revision}
\begin{revision}
We thank Minhao Liu for the \href{https://elvia-atlas.org/atlas/cases/neurips2025-0026-learning-cocoercive-conservative-denoisers-via-helmholtz-decomp.html}{ELVIA analysis}. This revision corrects the ADMM parameter range, the convergence assumptions, the Helmholtz interpretation, and Algorithm~\ref{alg power}. The corresponding changes are shown in blue. The reported experimental results are unchanged.
\end{revision}

\end{document}